\documentclass[12pt]{article}
\usepackage{amsmath,amssymb,amsthm,amsfonts,bm,mathrsfs,float,algorithm,algorithmic,caption,subcaption,setspace}
\usepackage{avant,array,fontenc,inputenc,natbib,hyperref,multirow,enumerate,rotating,booktabs,color,latexsym,times,stmaryrd}
\usepackage{url}  
\usepackage[ruled,algo2e]{algorithm2e}
\usepackage{dsfont}
\usepackage{bbm}
\usepackage{graphicx}
\graphicspath{{images/}}
\usepackage{longtable}
\usepackage{parskip}
\usepackage{enumerate}
\usepackage{hyperref}
\usepackage{wrapfig}
\usepackage{multirow}
\usepackage{hhline}
\usepackage{tabularx}
\usepackage{adjustbox}
\usepackage{pdflscape}
\usepackage{rotating}
\usepackage{array}
\usepackage{tikz}
\usepackage{float}
\usepackage{fancyvrb}
\usepackage{csquotes}
\usepackage{verbatim}
\usepackage{color}
\usepackage{soul}
\usepackage{booktabs}
\usepackage[hmargin=1in,vmargin=1in]{geometry}
\definecolor{blue}{HTML}{0000EE}
\definecolor{black}{HTML}{000000}
\hypersetup{
	linkcolor  = blue,
	citecolor  = blue,
	urlcolor   = black,
	colorlinks = true,
}
\usepackage[
open,
openlevel=2,
atend,
numbered
]{bookmark}

\usepackage{cleveref}

\crefformat{section}{\S#2#1#3} 
\crefformat{subsection}{\S#2#1#3}
\crefformat{subsubsection}{\S#2#1#3}
\usepackage{afterpage}

\def\RR{\mathds{R}}

\def\EE{\mathds{E}}

\def\DDD{\mathscr{D}}

\def\ZZZ{\mathscr{Z}}
\def\ind{\mathds{1}}

\newcommand{\cmmnt}[1]{\ignorespaces}

\numberwithin{equation}{section}
\newtheorem{lem}{Lemma}

\newtheorem{cond}{Condition}

\newcommand{\jeff}[1]{{#1}}

\DeclareMathOperator*{\argmax}{arg\,max}

\newcommand{\blind}{1}

\begin{document}
	
	\if1\blind
	{
		\title{\Large{\textbf{The Infinitesimal Jackknife and Combinations of Models}}}
		\author{
			\bigskip
			Indrayudh Ghosal\footnote{Denotes equal lead contributions.} $\;^1$, Yunzhe Zhou\footnotemark[1] $\;^2$, and Giles Hooker$\;^3$ \\
			\normalsize{\textit{$^1$ Department of Statistics and Data Science, Cornell University}} \footnote{Currently affiliated with Amazon.com Inc. All substantial contribution to this paper was prior to joining Amazon.com Inc.} \\
			\normalsize{\textit{$^2$ Department of Biostatistics, UC Berkeley}} \\
			\normalsize{\textit{$^3$ Department of Statistics, UC Berkeley}} \footnote{This work was partly conducted while Giles Hooker was visiting the Research School of Finance, Actuarial Studies and Statistics at the Australian National University.} 
		}
		\date{}
		\maketitle
	} \fi
	
	\if0\blind
	{
		\title{\Large{\textbf{The Infinitesimal Jackknife and Combinations of Models}}}
		\author{
			\bigskip
			\vspace{0.5in}
		}
		\date{}
		\maketitle
	} \fi

	\baselineskip=20pt
	\begin{abstract}
		The Infinitesimal Jackknife is a general method for estimating variances of parametric models, and more recently also for some ensemble methods. In this paper we extend the Infinitesimal Jackknife to estimate the covariance between any two models. This can be used to quantify uncertainty for combinations of models, or to construct test statistics for comparing different models or ensembles of models fitted using the same training dataset. Specific examples in this paper use boosted combinations of models like random forests and M-estimators. We also investigate its application on neural networks and ensembles of XGBoost models. We illustrate the efficacy of variance estimates through extensive simulations and its application to the Beijing Housing data, and demonstrate the theoretical consistency of the Infinitesimal Jackknife covariance estimate. 
	\end{abstract}
	
	
	\baselineskip=22pt

	
	\section{Introduction} \label{sec:intro}
	This paper focuses on methods that enable uncertainty quantification for combinations of models that are not immediately comparable. The jackknife is a resampling technique that can be used for bias and variance estimation [\cite{tukey1958bias}]. Here the estimator is  constructed by repeatedly calculating the statistic, each time leaving one observation from the sample out and averaging all estimates. As opposed to the ordinary jackknife, the Infinitesimal Jackknife (IJ) looks at the behavior of the statistic after giving an infinitesimally more weight to one observation (and down-weighting the rest to keep the total weight constant) [\cite{jaeckel1972infinitesimal}]. In recent work, the IJ has been widely used to quantify uncertainty for ensemble methods, most particularly for variance estimation for the predictions of random forests [\cite{wager2014confidence}, \cite{athey2019generalized}]. In addition, recent work has examined estimating the covariance between two random forests - specifically when they are used as base learners in a boosted setting [\cite{ghosal2020boosting}]. However, the definition of the IJ as proposed in \cite{efron1982jackknife} is much more general and can also be applied to a wide variety of parametric models. 
	
	This paper focuses on the application of IJ methods to combinations of models. Within machine learning, combinations of models arise from mixture models and from boosting methods [as explored in \cite{ghosal2020boosting}]. They may also arise if we wish to compare model predictions and provide uncertainty estimates for their difference. This may be used for assessing goodness of fit, or when assessing the impact of the choice of model or hyper-parameters.

	In this paper we demonstrate the use of the IJ for defining a general covariance between the predictions of any two models for which IJ estimates are available individually. We first derive IJ variance estimates for the predictions for three general model classes: ensemble models where each member of the ensemble is obtained using a subsample of the data, M-estimators based on optimizing a finite dimensional parameter vector, and kernel methods that use local averaging methods. We then demonstrate the use of IJ to explore combinations of these models, either to compare predictions between classes, or to combine models in a boosting-type procedure that generalizes the ideas in \cite{ghosal2020boosting}. We also examine extensions to small neural networks treated as M-estimators (in which find that the framework rapidly breaks down), and to ensembles of XGboost models where IJ estimates continue to perform well. 
	
	The paper is organized as follows. Because we will be concerned with the sample properties of confidence intervals and tests, we first introduce a common simulation framework and terminology for all our experiments in \cref{sec:sim_frame}. Then we give the definition of the IJ in \cref{sec:ij}, with the specific examples of  M-estimators, random forests, and kernel methods. In \cref{sec:boost}, we show how to use the IJ to calculate the confidence interval for the general boosting methods and local model modifications  (the latter mostly used for bias corrections). These are natural extensions of \cite{ghosal2020boosting} and \cite{ghosal2021generalised}. Then we demonstrate how it can be used to compare models in \cref{sec:modelcomp} with generalised linear models as a special case of M-estimators. In \cref{sec:nn}, we also explore how to derive the IJ for small neural network by regarding it as a M-estimator. This can make it possible to compare neural network with other models. We next discuss the general V-statistics framework to calculate the IJ for ensemble models and use an example of the ensemble XGBoost for illustrations in \cref{sec:v-stat}. Finally, we apply these methods to Beijing House Data in \cref{sec:real}. All technical theorems and proofs are provided in the Supplementary Appendix.

	\section{Simulation Framework} \label{sec:sim_frame}
	A key component of our contribution is to examine the many ways in which models may be combined. In order to assess the finite sample performance of our methods, we will use the following common simulation framework and terminology for all our experiments. 
	In our studies, we consider only the regression problems with response $Y|X=x \sim N(\eta(x),1)$ and $X \sim U[-1,1]^6$ with the following generative models (but we note that the theory we produce readily extends to an exponential family response):
	\begin{itemize}
		\item \textbf{Friedman}: $\eta(x) = 10\sin(\pi x_1 x_2) + 20(x_3 - \frac12)^2 + 10x_4 + 5x_5$.
		\item \textbf{Linear}: $\eta(x) = x_1 + x_2 + x_3 + x_4$
		\item \textbf{Constant}: $\eta(x) = 2$
	\end{itemize}
	
	For each simulation, we fix a set of 5 (for model comparisons) or 100 (for coverage calculations) query points at which our estimated function will be evaluated.  We then generate 200 data sets, each of 1000 points uniformly distributed on $[-1,1]^6$. This allows us to obtain a distribution of predictions at each of the query points, along with the covariance of predictions between them. The following metrics will be used throughout:
	\begin{itemize}
		\item \textbf{Coverage of Expectation  (CoE)}: Let $\hat{f}_{ij}$ denote the prediction at the $i$th query point and $j$th replication.  $\hat{V}_{ij}$ is its variance estimate. Then we construct 95\% confidence intervals by $\hat{f}_{ij} \pm \Phi^{-1}(0.975) \times \sqrt{\hat{V}_{ij}}$. For $j=1,2,...,200$ in each fixed query point $i$, we calculate a coverage probability by checking whether the expected prediction value (approximated by $\frac{1}{200}\sum_{j=1}^{200} \hat{f}_{ij}$) falls into this interval. We generate the boxplots and violin plots for the coverage probability of all the query points. \jeff{We use this to assess the accuracy of our estimates of variance. This is different from estimating model bias for which we use Coverage of Target defined below.}
		\item \textbf{Coverage of Target (CoT)}: Borrowing the notations above, we instead calculate the coverage probability by checking whether the true expectation $E(Y|X)$ value falls into the constructed 95\% confidence intervals. 
		\item \textbf{Coverage of Reproduction (CoR)}: In \cref{sec:real} we use and evaluate intervals that cover the value of an alternative $\hat{f}_{ij}$ obtained from an independent data set across 95\% of replications. 
	\end{itemize}
	We also summarize all the models that is considered in the simulations:
	\begin{itemize}
		\item \textbf{Random Forests}: Implemented by using \textsf{sklearn} package in python [\cite{scikit-learn}]. We vary the number
		of maximum depth in the range of $\{3, 5, 7, 9\}$ and also consider a random forest with trees grown to full depth. We set the number of trees to be 1000 or 5000 \jeff{and use a subsample size of 200, taken with replacement.}
		\item \textbf{Generalized Linear Models(GLM)}: We consider linear regression in this paper.
		\item \textbf{Neural Networks}: Implemented by using \textsf{Tensorflow2.0} package in python [\cite{abadi2016tensorflow}]. We vary the number
		of hidden units in the range of $\{1,3,5,10,20\}$ and consider ReLU and Sigmoid activation functions.
		\item \textbf{XGBoost}: We use the \textsf{xgboost} package for implementation and use the default hyperparameters [\cite{chen2016xgboost}].
	\end{itemize}

	\section{Infinitesimal Jackknife} \label{sec:ij}
	The IJ is a general-purpose framework for estimating the variance of any statistic. \cite{efron1982jackknife} defines the IJ for an estimate of the form $\hat{\theta}(P^0)$, where $P^0$ is the uniform probability distribution over the empirical dataset. We re-write our estimate more generally as $\hat{\theta}(P^*)$ where $P^*$ is any re-weighting of the empirical distribution. This is then approximated by the hyperplane tangent to the surface $\hat{\theta}(P^*)$ at the point $P^* = P^0$, i.e. $\hat{\theta}(P^*) \approx \hat{\theta}_{\text{TAN}}(P^*) = \hat{\theta}(P^0) + (P^*-P^0)^\top U$, where $U$ is a vector of the directional derivatives given by
	$$
	U_i = \lim_{\epsilon \to 0} \frac{\hat{\theta}(P^0 + \epsilon(\delta_i - P^0)) - \hat{\theta}(P^0)}{\epsilon}, i = 1,\ldots,n
	$$
	where $\delta_i$ the $i$th coordinate vector. Under a suitable asymptotic normal distribution for $P^*-P^0$ we can obtain the variance of $\hat{\theta}_{\text{TAN}}(P^*)$ to be 
	$$\text{Var}[\hat{\theta}_{\text{TAN}}(P^*)] = \frac{1}{n^2} \sum_{i=1}^n U_i^2$$
	This is the IJ variance estimate for the estimator $\hat{\theta}(P^0)$. Below we derive the form of IJ estimators for three specific models.
	
	\subsection{IJ for M-estimators} \label{subsec:ij_m}
	M-estimators are a broad class of extremum estimators for which the objective function is a sample average [\cite{hayashi2000extremum}]. Consider a data space $\ZZZ \subseteq \RR$ and the training data $Z_1,Z_2,\cdots,Z_n$ as i.i.d copies of $Z \in \ZZZ$. Then we define the M-estimator by
	$$
	\hat{\theta} = \argmax_{\theta \in \Theta} \EE_{\hat{\DDD}} [m(\theta, Z)]
	$$
	where $\EE_{\hat{\DDD}}$ is the expectation with regard to the empirical distribution $\hat{\DDD}$ over the training data, i.e., $\EE_{\hat{\DDD}} [m(\theta, Z)] = \frac1{n} \sum_{i=1}^n m(\theta, Z_i)$. $\Theta \subseteq \RR^p$ is the parameter space and $m(\theta, Z)$ is a well-behaved function from $\Theta \times \ZZZ$ to $\RR$. Since we are interested in the predictions of M-estimators, we denote by $\eta(\theta, x)$ the prediction function corresponding to any parameter $\theta \in \Theta$ and query point $x \in \ZZZ^d$. Then the directional derivatives of the model predictions are given by
	\begin{align} \label{M_ij}
		U_i(x) = -\nabla_\theta \eta(\hat{\theta}, x)^\top \left[\EE_{\hat{\DDD}} [\nabla_\theta^2 m(\hat{\theta},Z)]\right]^{-1} \nabla_\theta m(\hat{\theta}, Z_i)
	\end{align}
	Thus, the variance of the model predictions can be given by
	\begin{align*} 
		\text{Var}(\eta(\hat{\theta},x)) &= \frac{1}{n^2} \sum_{i=1}^n U^2_i(x) \\
		&= \frac{1}{n^2} \sum_{i=1}^n \nabla_\theta \eta(\hat{\theta}, x)^\top \left[\EE_{\hat{\DDD}} [\nabla_\theta^2 m(\hat{\theta},Z)]\right]^{-1} \times \\
		&\qquad\qquad \nabla_\theta m(\hat{\theta}, Z_i) \nabla_\theta m(\hat{\theta}, Z_i)^\top \left[\EE_{\hat{\DDD}} [\nabla_\theta^2 m(\hat{\theta},Z)]\right]^{-1} \nabla_\theta \eta(\hat{\theta}, x)^\top \\
		&= \nabla_\theta \eta(\hat{\theta}, x)^\top \left[\EE_{\hat{\DDD}} [\nabla_\theta^2 m(\hat{\theta},Z)]\right]^{-1}  \nabla_\theta \eta(\hat{\theta}, x)^\top   \\
		&\approx \nabla_\theta \eta(\theta^*, x)^\top \left[\EE_{\DDD} [\nabla_\theta^2 m(\theta^*,Z)]\right]^{-1}  \nabla_\theta \eta(\theta^*, x)^\top
	\end{align*} 
	\jeff{where $\theta^* = \argmax_{\theta \in \Theta} \EE [m(\theta, Z)]$ and $\EE_{\DDD}$ is the expectation with regard to the distribution of $Z$}. The approximation above holds for large samples due to the law of large numbers, the consistency of M-estimators and the continuity of $\eta$. This results in the usual plug-in values for the sandwich form of the asymptotic variance of $\eta(\hat{\theta},x)$.
	
	The derivation of these directional derivatives and consistency of corresponding variance estimate is shown in \cref{sec:MestIJ} and \cref{sec:IJcons}. In \cref{sec:nn}, we will further discuss how to formalize a small neural network model as an M-estimator and calculate its directional derivatives. 
	
	\subsection{IJ for Random Forests} \label{subsec:rf_bias}
%
	
	In \cite{efron2014estimation} directional derivatives are calculated for any general ensemble model. 
	
	Specifically, given the observations $z_1,z_2,\cdots,z_n$, we define the reweighted empirical probability distribution as
	$$\hat{F}^*: \text{mass} \; P^*_i \; \text{on} \; z_i, \quad  i=1,2,\cdots,n$$
	where the resampling vector $P^* \in \Big \{ \Big(P^*_1,P^*_2,\cdots,P^*_n \Big) \Big|P^*_i \geq 0, \sum_{i=1}^n P^*_i = 1 \Big \}$. In particular, $P^0 = (\frac{1}{n},\frac{1}{n},\cdots,\frac{1}{n})$.  We define the multinomial expectation of $\theta(P^*)$ as
	\begin{align*}
		S(P) = \mathbb{E}_*{[\hat{\theta}(P^*)]}, \quad P^*\sim \frac{\text{Mult}_n(n,P)}{n}, 
	\end{align*}
	for any resampling vector $P$ and  $\frac{\text{Mult}_n(n,P)}{n}$ is a rescaled multinomial distribution, that is $n$ independent draws on $n$ categories each having probability $1/n$. $\hat{\theta}(P^*)$ is the prediction of the random forest and $\mathbb{E}_*$ indicates expectation under the multinomial distribution of $P^*$. So $S(P^0)$ is the ideal smoothed bootstrap estimate.  Define the directional derivative
	\begin{align*}
		U_i = \lim_{\substack{\epsilon \to 0}} \frac{S(P^0 + \epsilon(\delta_i - P^0)) - S(P^0)}{\epsilon}, \quad i=1,\cdots,n 
	\end{align*}
	where $\delta_i$ the $i$th coordinate vector.  Formula (3.20) of \cite{efron2014estimation} gives
	\begin{align*}
		S(P^0 + \epsilon(\delta_i - P^0)) = S(P^0) + n \cdot (cov_*(N_{i,*}, T_*(x)))^2 + o(\epsilon)
	\end{align*}
	where $N_{i,*}$ is the number of times the $i$\textsuperscript{th} datapoint is in the ensemble, $T_*$ is the individual model in the ensemble and $cov_*$ represents the ideal covariance under the bootstrap sampling. This yields 
	$$
	U_i(x) = n \cdot cov_*(N_{i,*}, T_*(x))
	$$
	In applications, we generate finite number of bootstrap replications and use $$
	U_i(x) = n \cdot cov_b(N_{i,b}, T_b(x))
	$$ as the estimate instead. 
	
	For random forests the covariance is over $b = 1, \dots, B$ trees, $N_{i,b}$ is the number of times the $i$\textsuperscript{th} training data point is included in the sample used to train the $b$\textsuperscript{th} tree and $T_b$ is the $b$\textsuperscript{th} tree kernel. In \cref{sec:v-stat}, we will give an example of using XGBoost as a component in an ensemble model and using the above framework to calculate its directional derivatives.
	
	
	\cite{efron2014estimation} proved that the variance estimator using these directional derivatives is consistent for random forests where the training data for each tree is a full bootstrap resample of the original training data, whereas \cite{wager2018estimation} did the same where each tree uses a random subsample (without replacement) of the original training data.
	
	Following similar arguments as in Appendix B.1 of \cite{ghosal2020boosting} the theoretical variance for a random forest prediction at a query point $x$ can be given by $\frac{k_n^2}{n} \zeta_{1,k_n} + \frac1{B_n} \zeta_{k_n,k_n}$, where $\frac{k_n^2}{n} \zeta_{1,k_n}$ is estimated with the usual IJ directional derivatives, and $\zeta_{k_n,k_n} = var_*(T_*(x))$, i.e., the theoretical variance between the tree kernels $T_*$, can be estimated by $var_b(T_b(x))$.

	\subsubsection{Bias Corrections for IJ Estimates with Ensembles}
	Estimating $\frac{k_n^2}{n} \zeta_{1,k_n}$ with the infinitesimal Jackknife method introduces an upward bias (\cite{zhou2021v}, \cite{wager2018estimation}, \cite{wager2014confidence}). The source of this bias is  due to the fact that random forests are incomplete U-statistics, i.e, the number of trees used ($B_n$) are very small compared to the total number of possible trees that could be used (the theoretical $U_i$  assumes a model that uses every possible subsample whereas in practice we use a Monte Carlo estimate).  
	In practice $B_n \ll \binom{n}{k_n}$ for subsampling without replacement and $B_n \ll (k_n)^n$ for subsampling with replacement; requiring a correction term as discussed in below. 
	
	\begin{itemize}
		\item \textbf{Ranger's Correction}: Using the IJ directional derivatives defined above, the uncorrected variance estimate for a random forest is given by $\sum_{i=1}^n (cov_b(N_{i,b}, T_b(x)))^2$. But this is actually an estimate for the population quantity $\sum_{i=1}^n (cov_*(N_{i,*}, T_*(x)))^2$, where the covariance is over all possible trees. Focusing on the $i$\textsuperscript{th} term we are estimating $(cov_*(N_{i,*}, T_*(x)))^2$ with $(cov_b(N_{i,b}, T_b(x)))^2$. If we define $N^{(c)}$ and $T^{(c)}$ to be the mean-centered versions of the inbag count and trees respectively then $\left(\EE[N^{(c)}_{i,*} T^{(c)}_*(x)]\right)^2$ is estimated by $\left(\frac1{B_n} \displaystyle\sum_{b=1}^{B_n} N^{(c)}_{i,b} T^{(c)}_b(x)\right)^2$. When using a finite ensemble, we write
		$$
		\frac1{B_n} \displaystyle\sum_{b=1}^{B_n} N^{(c)}_{i,b} T^{(c)}_b(x) = \EE[N^{(c)}_{i,*} T^{(c)}_*(x)] + e,
		$$
		where $e$ is the Monte Carlo error. We assume that this error $e$ is independent of the population quantity $\EE[N^{(c)}_{i,*} T^{(c)}_*(x)]$, and thus $\left(\EE[N^{(c)}_{i,*} T^{(c)}_*(x)] + e\right)^2$ has an expected value of $\left(\EE[N^{(c)}_{i,*} T^{(c)}_*(x)]\right)^2 + \EE[e^2]$, where $\EE(e^2)$ can be estimated by 
		\begin{align*}
			\hat{\EE}[e^2] = \frac{var_b(N^{(c)}_{i,b} T^{(c)}_b(x))}{B_n} \approx \frac{var_b(N^{(c)}_{i,b}) \cdot var_b(T^{(c)}_b(x))}{B_n}    
		\end{align*}
		in which $N^{(c)}_{i,b}$ and $T^{(c)}_b(x)$ are assumed to be approximately independent. Finally note that $var_b(N^{(c)}_{i,b}) = var_b(N_{i,b})$ and $var_b(T^{(c)}_b(x)) = var_b(T_b(x))$. Hence we can estimate $\sum_{i=1}^n (cov_*{N_{i,*}, T_*(x)})^2$ by 
		$$
		\sum_{i=1}^n (cov_b(N_{i,b}, T_b(x)))^2 - \frac1{B_n} \sum_{i=1}^n var_b(N_{i,b}) \cdot var_b(T_b(x)).
		$$
		
		The final corrected variance estimate for a random forest prediction at a query point $x$ is thus given by
		$$
		\sum_{i=1}^n (cov_b(N_{i,b}, T_b(x)))^2 - \frac1{B_n} \left(\sum_{i=1}^n var_b(N_{i,b}) - 1\right) \cdot var_b(T_b(x))
		$$
		
		We can extend this correction to the covariance terms as well. For two query points $x_1$ and $x_2$ the original uncorrected covariance is given by
		$$
		\sum_{i=1}^n cov_b(N_{i,b}, T_b(x_1)) \cdot cov_b(N_{i,b}, T_b(x_2))
		$$
		and which is corrected by subtracting
		$$
		\frac1{B_n} \left(\sum_{i=1}^n var_b(N_{i,b}) - 1\right) \cdot cov_b(T_b(x_1), T_b(x_2)).
		$$
		
		We label this \textbf{``ranger's correction"} because this is the correction implemented in the \texttt{ranger} package for fitting random forests [\cite{rangerRpackage}]. While computationally simple, it relies on the strong underlying assumption that $N^{(c)}_{i,b}$ and $T^{(c)}_b(x)$ are independent, which introduces further bias for the corrected variance estimator. The constructed confidence interval based on this estimator could suffer from the under-coverage.
		
		\item \textbf{V-statistics based Correction}: 
		When the ensemble is obtained using subsamples taken with replacement, the random forests estimator can be regarded as a V-statistic [\cite{zhou2021v}]. Similar to \cite{sun2011efficient}, an ANOVA-like estimator of variance can then be constructed. Let $N_i = \sum_{b=1}^{B_n} N_{i,b}$, $m_i(x)= \sum_{b=1}^{B_n} \frac{N_{i,b}}{N_i} T_b(x)$ and $\bar{m}(x) = \frac{1}{n} \sum_{i=1}^n m_i(x)$. Define
		\begin{align*}
			\text{SS}_{\tau}(x) = \sum_{i=1}^n N_i (m_i(x) - \bar{m}(x))^2 \quad \text{and} \quad \text{SS}_{\epsilon}(x) = \sum_{i=1}^n \sum_{b=1}^{B_n} N_{i,b} (T_b(x) - m_i(x))^2
		\end{align*}
		A bias-corrected estimate can be given by
		$$
		\frac{\text{SS}_{\tau}(x) - (n-1){\hat{\sigma}}^2_{\epsilon}(x)}{C - \sum_{i=1}^n N^2_i/C},
		$$
		where $C = \sum_{i=1}^n N_i$ and ${\hat{\sigma}}^2_{\epsilon} (x) = \frac{\text{SS}_{\epsilon}(x)}{C-n}$.
		
		Similarly, this can be also extended to the case of covariance estimation as well.  For two query points $x_1$ and $x_2$, define
		\begin{align*}
			\text{SS}_{\tau}(x_1,x_2) &= \sum_{i=1}^n N_i (m_i(x_1) - \bar{m}(x_1)) (m_i(x_2) - \bar{m}(x_2)) \\ \text{SS}_{\epsilon}(x_1,x_2) &= \sum_{i=1}^n \sum_{b=1}^{B_n} N_{i,b} (T_b(x_1) - m_i(x_1))(T_b(x_2) - m_i(x_2))
		\end{align*}
		The bias-corrected estimate can be constructed by
		$$
		\frac{\text{SS}_{\tau}(x_1,x_2) - (n-1){\hat{\sigma}}^2_{\epsilon}(x_1,x_2)}{C - \sum_{i=1}^n N^2_i/C},
		$$
		where  ${\hat{\sigma}}^2_{\epsilon} (x_1,x_2) = \frac{\text{SS}_{\epsilon}(x_1,x_2)}{C-n}$ for the natural extensions of $\text{SS}_{\epsilon}$ and $\text{SS}_{\tau}$ to multiple query points. 
		
		This estimate is an unbiased (over the choice of subsamples) for the IJ estimate that results from using  all possible trees [\cite{searle2009variance}]. This framework can be also naturally extended to general ensemble models rather than just random forests and provides a unified treatment of variance estimation. In \cref{sec:v-stat}, we will also incorporate the ensemble XGBoost models into this framework. We point out that this method does not work for U-statistics for the reason that its sampling schema is not equivalent to sampling from the empirical distribution. There is also work that develops an unbiased variance estimator based on incomplete U-statistics for variance estimation of random forests, where the tree size is allowed to be comparable with the sample size [\cite{xu2022variance}]. However, we only focus on the case when the subsampling is taken with replacement in this paper.
	\end{itemize}
	
	In order to illustrate the advantage of V-statistics based correction over ranger's correction, we present a small simulation to compare their ability to estimate variance by examining the coverage of $E\hat{f}(x)$ (i.e. CoE) for predictions by using the experiment settings mentioned in \cref{sec:sim_frame}. We consider three data generating processes. We use the random forests with trees grown to full depth and set the number of trees B to be 1000 or 5000. The result is presented in Figure \ref{fig_bias}. We draw the boxplots and violin plots for the CoE over 100 query points. From the plots, we observe that ranger's correction suffers from undercoverage, especially when the number of trees is not very large. In contrast, V-statistics based correction can produce coverage probability close to or above 0.95 with reasonable number of trees. For the remainder of this paper, we will use V-statistics based correction for estimating the covariance matrix of random forests.
	
	\begin{figure}[ht!]
		\centering
		\includegraphics[width=0.9\textwidth]{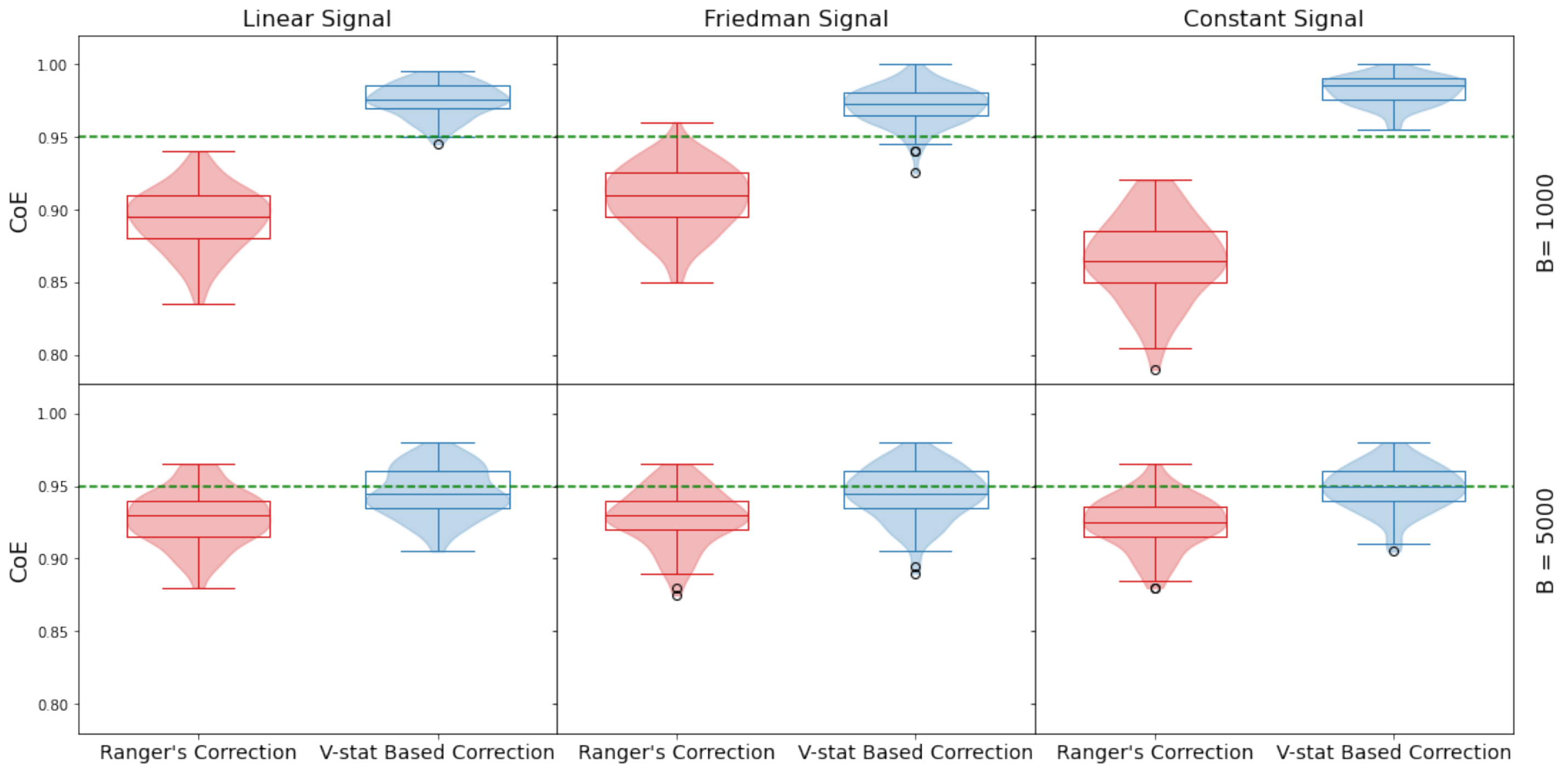}		
		\captionsetup{width=0.9\textwidth}
		\caption{Comparison of coverage of $E[\hat{f}(x)]$  using 95\% confidence intervals between ranger's correction (red) and the V-statistic based correction (blue). We consider three data generating processes. The number of trees $B$ is set to be 1000 or 5000. Plots present results for coverage at 100 randomly generated query points.} 
		\label{fig_bias}
	\end{figure}

	\subsection{IJ for Kernel Methods}\label{subsec:kernel}
	Kernel regression is a non-parametric technique to estimate the conditional expectation of a random variable $Y$ relative to a variable $X$ [\cite{nadaraya1964estimating,watson1964smooth}]. Given the dataset $\{X_k,Y_k\}_{k=1}^n$, the Nadaraya–Watson estimator for any query point $x$ is defined as
	\begin{align*}
		\hat{m}_h(x) = \frac{\sum_{k=1}^nK_n(x-X_k)Y_k}{\sum_{k=1}^nK_n(x-X_k)}
	\end{align*}
	We define $\hat{m}_h(x) = \frac{\sum_{k=1}^n C_k Y_k}{\sum_{k=1}^n C_k}$, where $C_k = K_n(x-X_k)$. Then the directional derivatives of $\hat{m}_h(x)$ are given by
	$$
	U_i = \lim_{\epsilon \to 0} \frac{\hat{\theta}(P^0 + \epsilon(\delta_i - P^0)) - \hat{\theta}(P^0)}{\epsilon}, \text{ where } \hat{\theta}(P) = \frac{\sum_{k=1}^n P_k C_k Y_k}{\sum_{k=1}^n P_k C_k}, \sum_{k=1}^n P_k = 1
	$$
	We let $p = \sum_{k=1}^n C_k Y_k$, $q = \sum_{k=1}^n C_k$, $r = nC_iY_i$ and $s = nC_i$ then we get 
	\begin{align*}
		U_i &= \lim_{\epsilon \to 0} \frac1\epsilon \left[ \frac{\frac1{n}(1-\epsilon+n\epsilon)C_iY_i + \sum_{k \neq i} \frac1{n} (1-\epsilon)C_kY_k}{\frac1{n}(1-\epsilon+n\epsilon)C_i + \sum_{k \neq i} \frac1{n} (1-\epsilon)C_k} - \frac{\sum_{k=1}^n C_k Y_k}{\sum_{k=1}^n C_k} \right] \\
		&= \lim_{\epsilon \to 0} \frac1\epsilon \left[ \frac{n\epsilon C_iY_i + (1-\epsilon)\sum_{k=1}^n C_k Y_k}{n\epsilon C_i + (1-\epsilon)\sum_{k=1}^n C_k} - \frac{\sum_{k=1}^n C_k Y_k}{\sum_{k=1}^n C_k} \right] \\
		&= \lim_{\epsilon \to 0} \frac1\epsilon \left[ \frac{(1-\epsilon)p + \epsilon r}{(1-\epsilon)q + \epsilon s} - \frac{p}{q} \right] \\
		&= \lim_{\epsilon \to 0} \frac{(1-\epsilon)pq + \epsilon rq - (1-\epsilon) pq - \epsilon ps}{\epsilon q ((1-\epsilon)q + \epsilon s)} \\
		&= \frac{rq - ps}{q^2}
	\end{align*}
	With the directional derivatives above, we can calculate variance estimates and produce confidence intervals  for kernel regression predictions. In \cref{subsec:modif}, we will show that the local model modifications proposed in [\cite{lu2021unified}] share similar structures with kernel regression so we can follow the calculations above to obtain their directional derivatives.

	\section{Boosting Methods} \label{sec:boost}
	If a model is inadequate to capture the underlying relationships between $y$ and $x$ in a data set, another model can be added to pick up the rest of the signal. We define this process to be be boosting in reference to gradient boosting [\cite{friedman2001greedy}]. Previous work in \cite{ghosal2020boosting} has shown that boosting a random forest with another one achieves significant reduction in bias and test-set MSE with a small increase in variance. An extension of this can be applied to non-Gaussian signals where we initialise with an MLE type estimator and then boost with (one or two) random forests with appropriately defined ``pseudo-residuals'' as the training signals [\cite{ghosal2021generalised}]. We further extend the idea in this paper to boost any model with another one and obtain uncertainty quantification via the IJ covariance estimate.
	
	A common problem shared by this boosting technique is the inheritance of variability from one stage of modelling to the next. This does not impede the modelling process on empirical data but asymptotic analysis of these boosted models are easier if we assume the following condition.
	
	\begin{cond} \label{cond:regularity}
		We're given an initial dataset $D = (Z_i)_{i=1}^n = (Y_i, X_i)_{i=1}^n$. Suppose $g$ is a fixed function of $X$ (the predictor) and $\hat{g}$ is an unbiased estimator of $g$ based on this dataset. Assume  a model $\hat{f}$ is fit with training signal $\big(h(\hat{g}(X_i), Z_i)\big)_{i=1}^n$. If $\check{f}$ is the model fitted the same way as $\hat{f}$ but with training signals $\big(h(g(X_i), Z_i)\big)_{i=1}^n$ then for any query point $x$
		$$
		\frac{\hat{f}(x) - \check{f}(x)}{\sqrt{var(\hat{f}(x))}} \xrightarrow{p} 0.
		$$
	\end{cond}
	The function $h$ may include training weights if needed. Special cases of this condition has been used in previous literature:
	\begin{itemize}
		\item In \cite{ghosal2020boosting} $\hat{g}$ and $\hat{f}$ are the first and second (base and boost) stage random forests respectively.
		\item In \cite{ghosal2021generalised} $\hat{g}$ could be the MLE-type estimator ($\hat{\eta}_{MLE}^{(0)}$) or the sum of $\hat{\eta}_{MLE}^{(0)}$ and the first random forest ($\hat{f}_1$) - in those cases $h$ would be the training signals and weights for the random forests $\hat{f}_1$ and $\hat{f}_2$ respectively. 
	\end{itemize}

	For two-step boosted models our estimators will be of the form $\hat{f}_1 = \hat{F}_1$ and $\hat{f}_2 = \hat{F}_1 + \hat{F}_2$. Suppose in this case the directional derivatives for $\hat{F}_1$ and $\hat{F}_2$ are given by $U_i^{(1)}(x)$ and $U_i^{(2)}(x)$ respectively for any query point $x$ and $i = 1, \dots, n$ and also define
	$$
	\hat{\Sigma}^{(pq)}_{ij} = \frac1{n^2} \sum_{k=1}^n U_k^{(p)}(x_i) U_k^{(q)}(x_j), \; p,q = 1,2, \; i,j = 1, \dots, m
	$$
	to be the covariance estimate for $\Sigma^{(pq)}_{ij} := cov(\hat{F}_p(x_i), \hat{F}_q(x_j))$. Then we can construct valid confidence interval for $\hat{f}_2$. Specifically, the covariance matrix of $\hat{f}_2$ over the query points will be given by $\hat{\Sigma} = \hat{\Sigma}^{(11)} + \hat{\Sigma}^{(22)} + \hat{\Sigma}^{(12)} + \hat{\Sigma}^{(21)}$ which can then be used to construct confidence intervals.
	
	\subsection{Boosted Forests} \label{subsec:boost_rf}
	\cite{ghosal2020boosting} discussed a version of boosting models for Gaussian responses. There the first stage model (base learner) was a random forest and the subsequent boosted stage model was another random forest whose training signals were the residuals of the first forest. Here, boosted forests provided lower test set MSE and higher confidence interval coverage. With the techniques of this paper, we compare coverage of $f(x)$ (i.e. CoT) for prediction intervals between random forests and boosted forest models.
	
	We use the same simulation setup as in \cref{sec:sim_frame}. Figure \ref{fig_boost_rf} presents boxplots and violin plots of the coverage of $f(x)$ for 95\% confidence intervals. We use a random forest with the maximum depth of 3 as the initial estimator at the first step. In the second step, we use random forests with different maximum depth (3,5,7,9, "Full" on the x axis) for boosting, where "Full" represents  random forests with trees grown to full depth. We also include the case when there is no second step of boosting (red boxplot and violin plot in each subplot). For comparison purpose, we also use green star symbols to present the average coverage of $E\hat{f}(x)$ (i.e. CoE) with 95\% level over 100 query points. 
	
	\begin{figure}[ht!]
		\centering
		\includegraphics[width=0.9\textwidth]{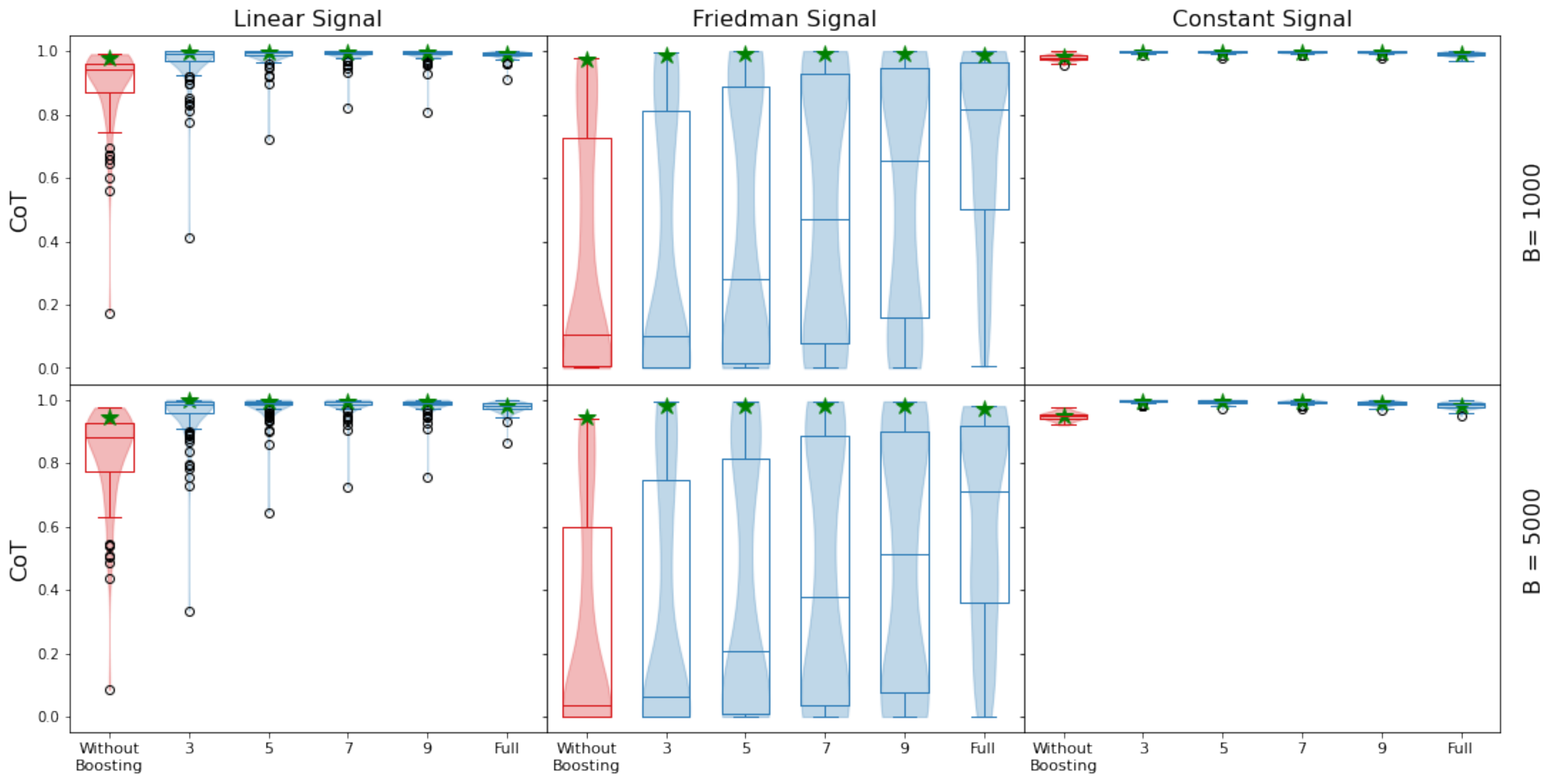}
		\captionsetup{width=0.9\textwidth}
		\caption{Coverage of $f(x)$ (CoT) of 95\% confidence intervals for boosted of random forests taken over 100 query points. Three scenarios of model signals are considered and the number of trees $B$ is set to be 1000 or 5000. We use a random forest with the maximum depth of 3 as the initial estimator at the first step. In each subplot, the x axis represents the model for the second boosting step. The red boxplot  and violin plot with x axis label "Without Boosting" is for the case when there is no second step of boosting. Numbers "3,5,7,9" denote the different values of maximum depth when using random forests for boosting. "Full" represents random forests with trees grown to full depth for boosting. For comparison, we also use green stars to present the average coverage of $E \hat{f}(x)$  (CoE) over the same query points.} 
		\label{fig_boost_rf}
	\end{figure}
	
	A first observation is that under Linear and Friedman signals, the initial random forests model suffers from low coverage of the target and the boosting process can efficiently improve it. This tells us that the initial model has serious bias and this bias can be further reduced by using boosting methods. Particularly for the Friedman signal, we can see that increasing the maximum depth in the second boosting step can lead to a higher coverage rate. By contrast, in the Constant signal, we can see that the initial model has already achieved a good coverage rate. As a result, boosting with a second model is unnecessary. Coverage of $E\hat{f}(x)$ is consistently good across all the settings indicating that our IJ methods capture the variance well. 
	
	\subsection{Forest Refinements of a Linear Model}
	An important observation within this paper is that the generality of the IJ estimate means that we need not restrict ourselves to boosting using only ensembles. Here we work with the GLM (generalized linear model) class of models to which we add a random forest to capture remaining signal and compare the performance of such a boosted model with the base GLM. We use linear model as a specific example of GLM for the demonstrations.
	
	We consider the exact same setting in \cref{subsec:boost_rf}, except that we use simple linear regression model for the initial step of boosting. The resulting coverage is shown in Figure \ref{fig_boost_lm}. Clearly in case of the Linear and Constant signals the linear model performs optimally making the boosting step unnecessary. Also as expected the linear model doesn't perform well for the Friedman signal and it's forest refinement can achieve better coverage. In \cref{sec:modelcomp}, we will demonstrate that we can also look at tests to check whether boosting gives a significant difference for the model predictions.
	
	\begin{figure}[ht!]
		\centering
		\includegraphics[width=0.9\textwidth]{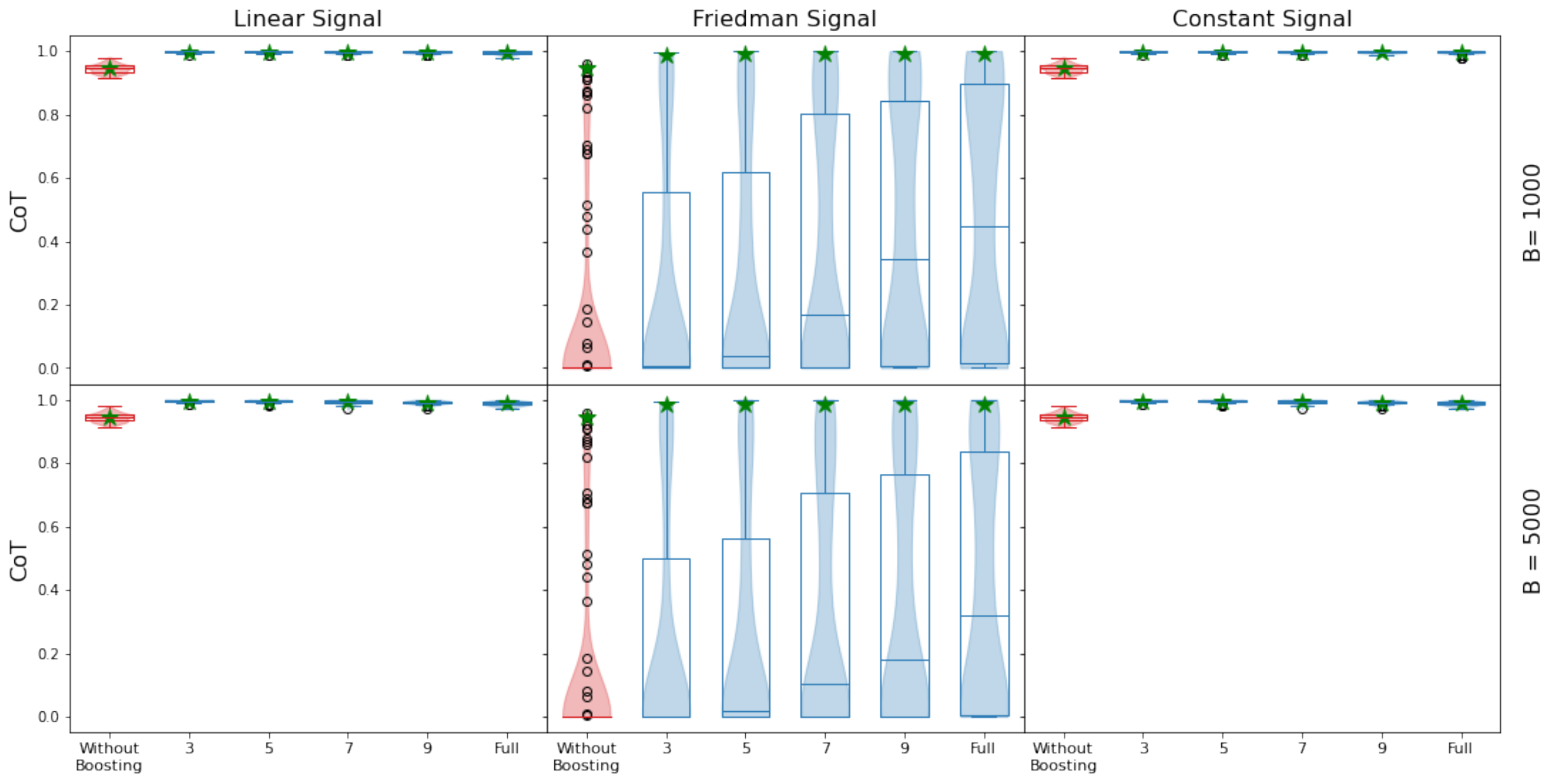}
		\captionsetup{width=0.9\textwidth}
		\caption{Coverage of $f(x)$ using 95\% confidence intervals for boosting a linear model with random forests over 100 query points. The rest of the setup is the same as Figure \ref{fig_boost_rf}.} 
		\label{fig_boost_lm}
	\end{figure}

	\subsection{Local Model Modifications of Random Forests} \label{subsec:modif}
	The approach of boosting models above can be cast within the framework of smoothing residuals to test for goodness of fit [\cite{hart2013nonparametric}], and this can also be thought of as providing a bias correction by boosting with a second-stage kernel smoother. We illustrate this procedure with the following example method.
	
	In \cite{lu2021unified} the a local linear modification is defined by a bias term is given by $\widehat{Bias}(x) = \sum_{k=1}^n v_k(x) (Y_k - \hat{f}(X_k))$, where $(Y_k, X_k)$ is the $k$\textsuperscript{th} training data, $x$ is a query point and $v_k(x)$ is the out-of-bag in-leaf proportion given by
	$$
	v_k(x) = \frac{\sum_{b=1}^{B_n} \ind\left\{ Z_k \notin I_b, X_k \in L_b(x) \right\}}{\sum_{\ell=1}^n \sum_{b=1}^{B_n} \ind\left\{ Z_\ell \notin I_b, X_\ell \in L_b(x) \right\}}
	$$
	in which $X_k \in L_b(x)$ indicates that $X_k$ falls into the same leaf as $x$ in the $b$th tree, and we weight co-inclusion only over out-of-bag samples. 
	In that paper, prediction intervals are produced from the distribution of out-of-bag errors weighted by the $v_k$. In this paper we use the IJ variance estimates to provide confidence intervals of the bias corrected estimator. We rewrite $\widehat{Bias}(x) = \frac{\sum_{k=1}^n C_k D_k}{\sum_{k=1}^n C_k}$, where $C_k = \sum_{b=1}^{B_n} \ind\left\{ Z_k \notin I_b, X_k \in L_b(x) \right\}$ and $D_k = Y_k - \hat{f}(X_k)$. If we denote $p = \sum_{k=1}^n C_k D_k$, $q = \sum_{k=1}^n C_k$, $r = nC_iD_i$ and $s = nC_i$, then following the same calculations in \cref{subsec:kernel},  the directional derivatives of $\widehat{Bias}(x)$ are given by
	\begin{align*}
		U_i = \frac{rq - ps}{q^2}
	\end{align*}
	Then the final estimate given by $\hat{f}(x) + \widehat{Bias}(x)$ will have a variance that is consistently estimated by
	$$
	V(x) = \frac1{n^2} \sum_{i=1}^n (U'_i + U_i)^2, \qquad U'_i = n \cdot cov_b(N_{i,b}, T_b(x))
	$$
	Using this variance estimate we can thus compare the models $\hat{f}$ vs $\hat{f} + \widehat{Bias}$ and also produce confidence intervals for the final estimator $\hat{f} + \widehat{Bias}$.
	
	Using the simulation setup in \cref{sec:sim_frame}, we compare the coverage of $f(x)$ (CoT) between the original model and the model after modification. The resulting plot is in Figure \ref{fig_mod_rf}. We observe that the local linear modifications does not affect coverage for the Constant signal since the random forest with a simple structure is already enough to fit the signal. As for Linear and Friedman signal, we can see that local model modification can efficiently improve the coverage rate by reducing the bias of original model estimator. 
	
	We notice that Figure \ref{fig_boost_rf}, \ref{fig_boost_lm} and \ref{fig_mod_rf} share similar patterns. Under the Constant signal, the original model (without boosting or local linear modifications) already achieves good CoT. So boosting or model modifications doesn't change the coverage much at all. However, under the Friedman signal, the original model suffers from bias and the boosted or modified model can efficiently improve the coverage. Under Linear signal, the original linear regression model of Figure \ref{fig_boost_lm} already achieves optimal performance while it still needs boosting or modifications in the scenarios of Figure \ref{fig_boost_rf} and \ref{fig_mod_rf}. By comparing the CoT, we conclude that boosting of random forests in Figure \ref{fig_boost_rf} can achieve the best overall performance.
	
	\begin{figure}[ht!]
		\centering
		\includegraphics[width=0.9\textwidth]{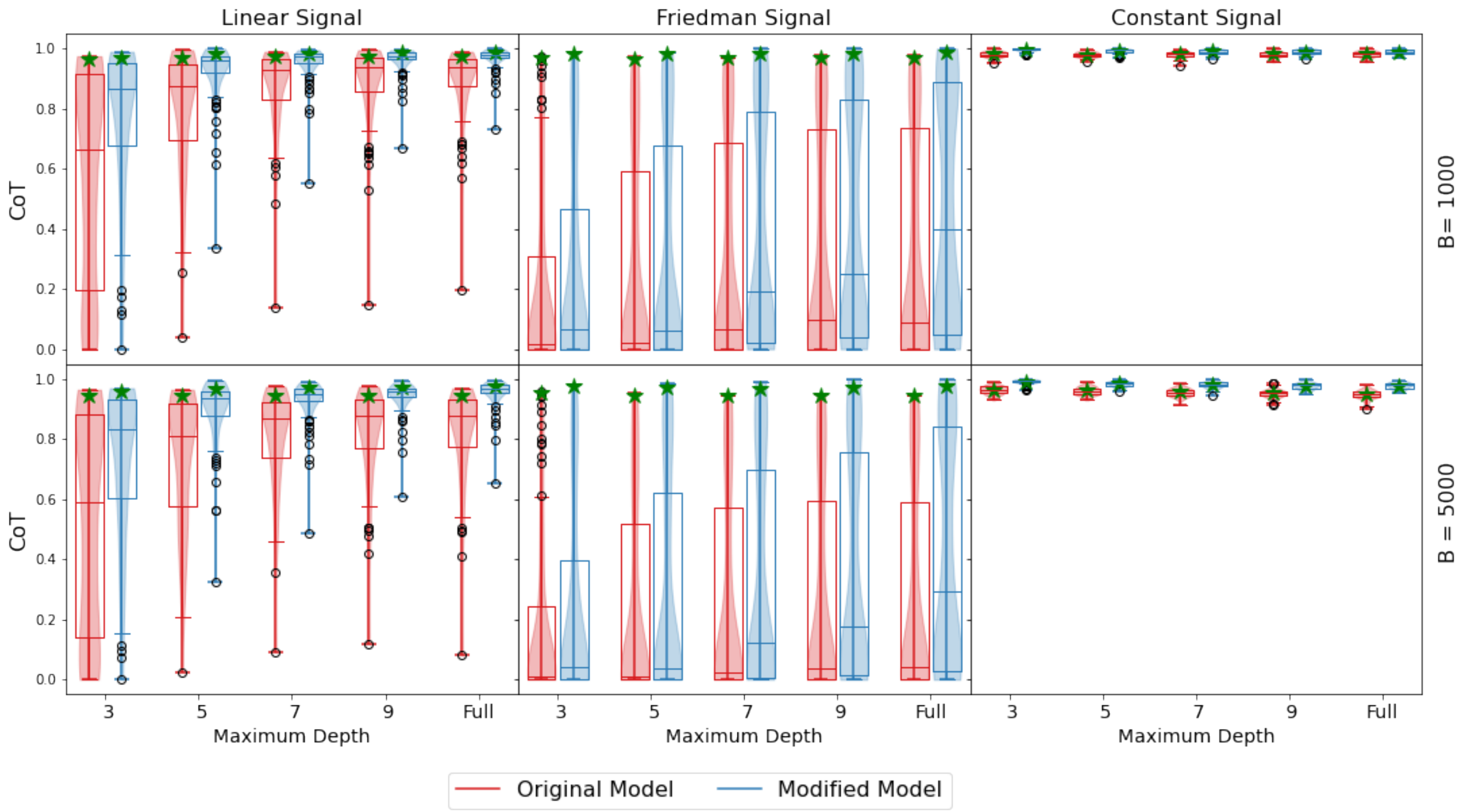}
		\captionsetup{width=0.9\textwidth}
		\caption{Coverage of $f(x)$ (CoT) using 95\% confidence intervals for local model modification of random forests over 100 query points. The rest of the setup is the same as Figure \ref{fig_boost_rf}.} 
		\label{fig_mod_rf}
	\end{figure}

	\section{Model Comparisons} \label{sec:modelcomp}
	In this section, we first demonstrate the procedure of constructing the test statistics for model comparisons. Then we give an example by comparing random forests with different maximum depths and GLM. 
	
	\subsection{Test Statistics Construction} \label{subsec:testcomp}
	Suppose we have training data $\big(Z_i\big)_{i=1}^n = \big((Y_i,X_i)\big)_{i=1}^n$ and we fit two estimators $\hat{f}_1$ and $\hat{f}_2$. We wish to test if the predictions from these two estimators are significantly different. Note that this difference (and its significance) may change with the query point. This will be demonstrated over a set of query points $x_1, \dots, x_m$; evaluation of this difference at a single query point is a special case.
	
	Suppose the directional derivatives for $\hat{f}_1$ and $\hat{f}_2$ are given by $U_i^{(1)}(x)$ and $U_i^{(2)}(x)$ respectively for any query point $x$ and $i = 1, \dots, n$. Now consider the vector
	$$
	F = \begin{pmatrix} F_1 \\ F_2 \end{pmatrix} \sim N(\mu, \Sigma) ; \text{ asymptotically, where } F_p = \begin{pmatrix} \hat{f}_p(x_1) \\ \vdots \\ \hat{f}_p(x_m) \end{pmatrix}, \; p = 1,2
	$$
	here $\Sigma = \begin{pmatrix}
		\Sigma^{(11)} & \Sigma^{(12)} \\
		\Sigma^{(21)} & \Sigma^{(22)}
	\end{pmatrix}$ is a $2m \times 2m$ with $m \times m$ blocks. Now $\Sigma^{(pq)}_{ij}$ is the covariance between $\hat{f}_p(x_i)$ and $\hat{f}_q(x_j)$, $p,q = 1,2$, $i,j = 1, \dots, m$ and can be estimated by 
	$$
	\hat{\Sigma}^{(pq)}_{ij} = \frac1{n^2} \sum_{k=1}^n U_k^{(p)}(x_i) U_k^{(q)}(x_j)
	$$
	
	The consistency of the IJ covariance estimate between two random forests was shown in \cite{ghosal2020boosting}. Consistency of the covariance estimate between an M-estimator and a random forest is shown in \cref{sec:IJcons}; a special case of this consistency was shown previously in \cite{ghosal2021generalised} where the M-estimator was an MLE.
	
	Now our null hypothesis is that the estimators $\hat{f}_1$ and $\hat{f}_2$ are not significantly different and thus $\EE[F_1 - F_2] = 0$ under the null. Further the covariance matrix of $F_1 - F_2$ will be consistently estimated by $\hat{\Sigma} = \hat{\Sigma}^{(22)} + \hat{\Sigma}^{(11)} - \hat{\Sigma}^{(12)} - \hat{\Sigma}^{(21)}$. Thus our test-statistic will be
	$$
	(F_1 - F_2)^\top \hat{\Sigma}^{-1} (F_1 - F_2) \sim \chi^2_m
	$$
	where $\chi^2_m$ represents the chi-squared distribution with $m$ degrees of freedom. Note that the estimation of $\tilde{\Sigma} = \Sigma^{(22)} + \Sigma^{(11)} - \Sigma^{(12)} - \Sigma^{(21)}$ with $\hat{\Sigma}$ involves uncertainty which will depend on $n$ and other details of fitting the estimators $\hat{f}_1$ and $\hat{f}_2$ such as number of trees for fitting random forests. This uncertainty is unaccounted for in our procedures; we assume that it is small enough so that the null distribution of the test-statistic above holds asymptotically. 
	
	In the case of random forests, note that we can directly use V-statistics based correction method to obtain unbiased estimates of $\Sigma^{(11)}$ and $\Sigma^{(22)}$. Then $\Sigma^{(12)}$ and $\Sigma^{(21)}$ are estimated by using the uncorrected directional derivatives $U_i(x) = n \cdot cov_b(N_{i,b}, T_b(x))$ defined in \cref{subsec:rf_bias}. This covariance estimate is unbiassed when $f_1$ and $f_2$ are independent, conditional on the training set since the Monte Carlo error in the directional derivatives of each has expectation 0.  When $f_1$ and $f_2$ are both ensembles, they can alternatively be obtained using the same subsamples, in which case the combination of corresponding ensemble members can be thought of as a single ensemble and IJ applied accordingly. \cite{ghosal2020boosting} found that choosing subsamples independently between the two models improved predictive performance we we take this approach throughout.

	When a large number of query points are used, we must estimate a covariance matrix with large dimensions. Doing so accurately becomes increasingly challenging as the dimension increases. These errors are compounded when taking the difference between models and readily result in negative eigenvalues. For this paper we only use a small $(\leq 5)$ number of qu this to comparing any two stages of a boosting model if we boost for more than two steps). Using the notations in \cref{sec:boost}, we wish to test if the use of $\hat{f}_2$ was a signifiery points to ensure a stable and well-conditioned $\Sigma$, but discuss potential alternatives in \cref{sec:dis}.
	
	Note that for boosting models we can check if the boosting provides a significant difference by comparing the base model and the final estimate (in fact we can generalise this to comparing any two stages of a boosting model if we boost for more than two steps). Using the notations in \cref{sec:boost}, we wish to test if the use of $\hat{f}_2$ was a significant addition, i.e, if the predictions from $\hat{f}_1$ and $\hat{f}_1+\hat{f}_2$ are significantly different.
	
	In this case 
	$$
	F_p = \begin{pmatrix} \hat{g}_p(x_1) \\ \vdots \\ \hat{g}_p(x_m) \end{pmatrix}, p = 1,2 \text{ and } \hat{\Sigma}^{(pq)}_{ij} = \frac1{n^2} \sum_{k=1}^n V_k^{(p)}(x_i) V_k^{(q)}(x_j),
	$$
	where $\hat{g}_1 = \hat{f}_1$, $\hat{g}_2 = \hat{f}_1 + \hat{f}_2$, $V_k^{(1)} = U_k^{(1)}$, $V_k^{(2)} = U_k^{(1)} + U_k^{(2)}$. Thus our test-statistic can be constructed as
	$$
	G^\top \left( \hat{\sigma}^{(22)} \right)^{-1} G \sim \chi^2_m
	$$
	where 
	\begin{align*}
		G &:= F_2 - F_1 =  \begin{pmatrix} \hat{f}_2(x_1) \\ \vdots \\ \hat{f}_2(x_m) \end{pmatrix} \\
		\hat{\sigma}^{(22)}_{ij} &:= \hat{\Sigma}^{(22)}_{ij} + \hat{\Sigma}^{(11)}_{ij} - \hat{\Sigma}^{(12)}_{ij} - \hat{\Sigma}^{(21)}_{ij} = \frac1{n^2} \sum_{k=1}^n U_k^{(2)}(x_i) U_k^{(2)}(x_j),
	\end{align*}
	Similarly as above the estimate $\hat{\sigma}^{(22)}$ has degrees of freedom which will depend on $n$ and other details of fitting the estimators $\hat{f}_1$ and $\hat{f}_2$. These degrees of freedom are difficult to calculate but we can assume that it is high enough so that the null distribution of the test-statistic above holds asymptotically.
	
	\subsection{Comparisons between random forests and a Linear Model} \label{subsec:comp_rf}
	We give an example of comparing random forests and linear model. We consider the settings described in \cref{sec:sim_frame} and set the number of query points to be 5. Following the procedure in \cref{subsec:testcomp}, we compare the select pairs of models and calculate the power, i.e., the proportion of times the null hypothesis was rejected out of 200 replicates at a 5\% level. These power values are graphically presented in Figure \ref{fig_com_rf}. In each subplot, the numbers $\{3,5,7,9\}$ on the x axis represent random forests with different maximum depth. And "Full" represents random forests with trees grown to full depth. "LM" on the x axis represents the simple linear regression. The colors in each heatmap plot denote the strength of the power. A darker rectangle between two models means a higher power of rejecting the null hypothesis. 
	
	\begin{figure}[ht!]
		\centering
		\includegraphics[width=0.9\textwidth]{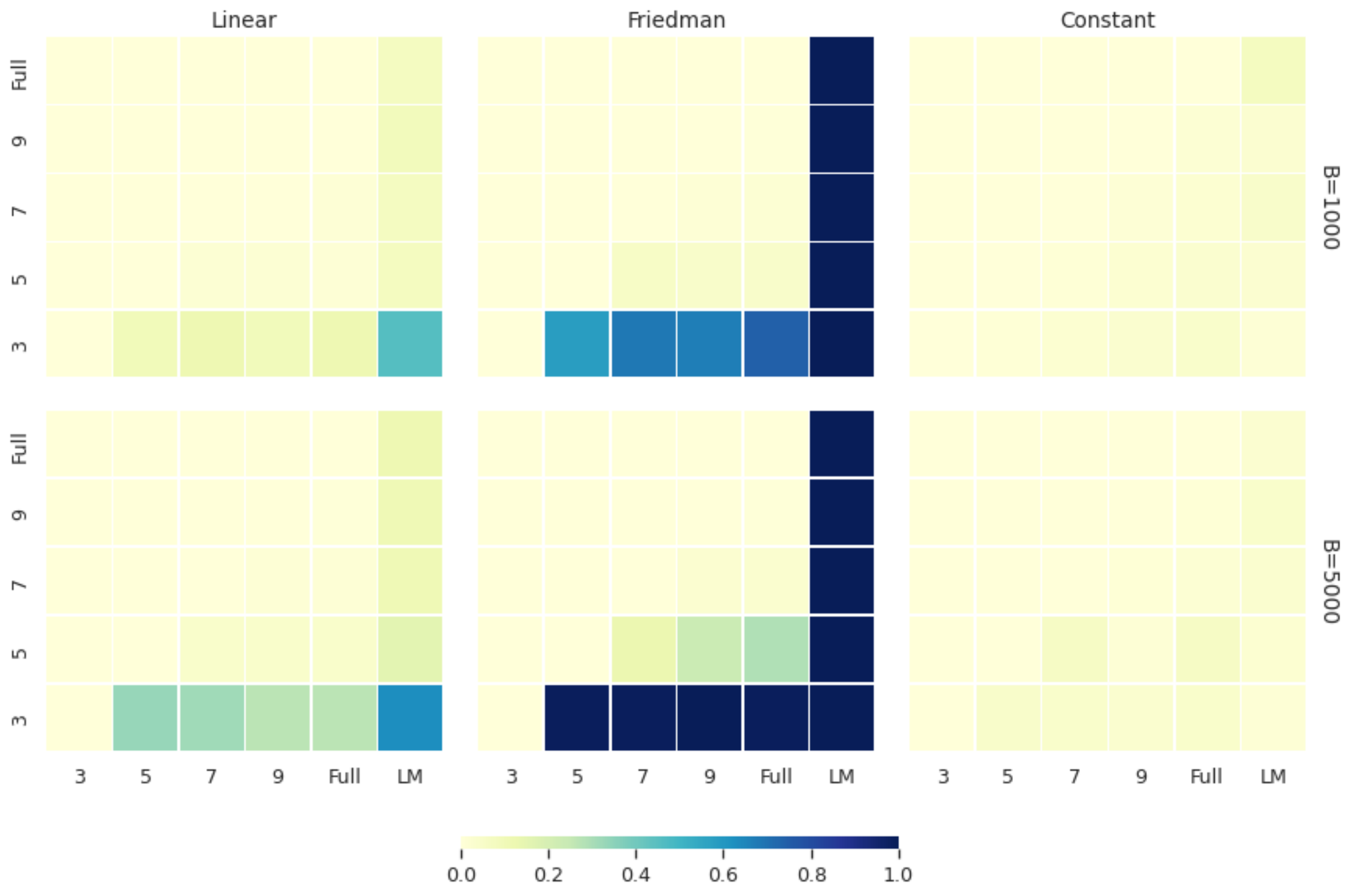}
		\captionsetup{width=0.9\textwidth}
		\caption{Model Comparisons of multiple random forests and linear model. Three data generating processes are considered and the number of query points is set to be 5. The number of trees $B$ is set to be 1000 or 5000. In each subplot, we draw the heatmap for the power, i.e., the proportion of times the null hypothesis was rejected out of 200 replicates at a 5\% level between each model pair. Numbers $\{3,5,7,9\}$ on the x or y axis represent random forests with different maximum depth. "Full" indicates random forests with trees grown to full depth. "LM" on the x axis represents the linear model and we consider simple linear regression in this example.} \label{fig_com_rf}
	\end{figure}
	
	We see that for the Friedman signal the linear model is found to be significantly different from all the random forest models, whereas the effect is much less pronounced for the Linear and Constant signals. The comparison between the different forests shows nearly no difference for Constant signal. This is reasonable since the Constant signal can be easily fitted with any maximum depth. As for Linear and Friedman signals, we observe a relatively high power between the value "3" and "5,7,9, Full". However, the differences attenuate as the maximum depth increases. The choice of larger number of trees gives us more significant differences. This is because we can get a more precise estimator of the covariance with V-statistics based correction. A relatively small number of trees could possibly lead to a mildly conservative test. 
	
	We also give an example to demonstrate the comparisons of boosting models in Figure \ref{fig_com_boost}. We boost linear model with random forests with different maximum depth and compare the base model and the final boosted estimate to check if the boosting provides a significant difference. We observe that only under Friedman Signal, the difference is very significant. However, it shows nearly no difference for both Linear and Constant Signals.
	
	\begin{figure}[ht!]
		\centering
		\includegraphics[width=0.9\textwidth]{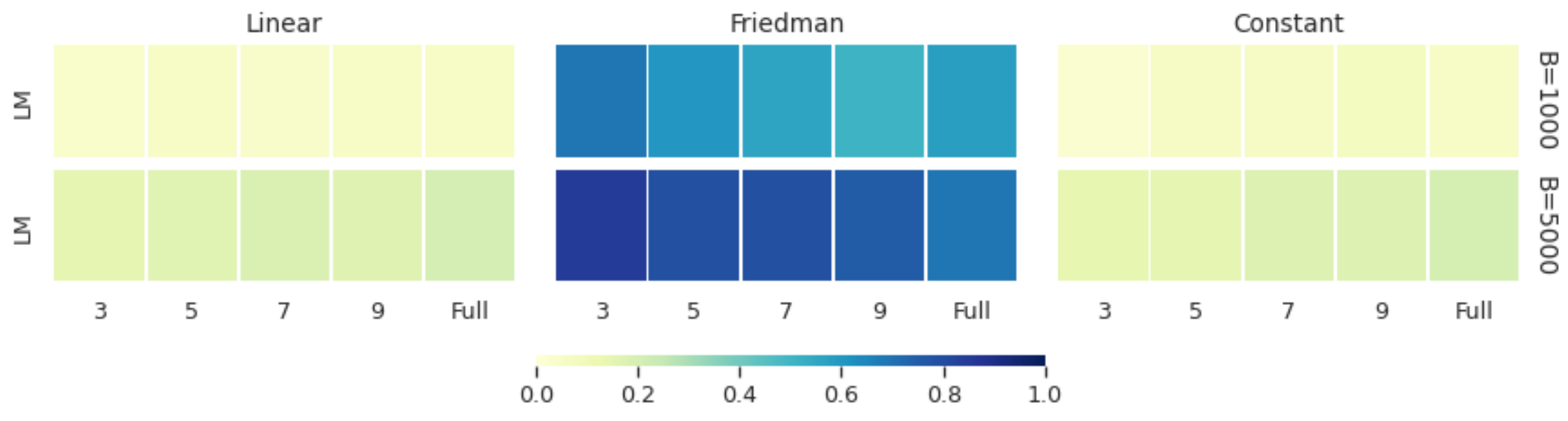}
		\captionsetup{width=0.9\textwidth}
		\caption{Model Comparisons for boosting LM with random forests.Three data generating processes are considered and the number of query points is set to be 5. The rest of the setup is the same as Figure \ref{fig_com_rf}.}
		\label{fig_com_boost}
	\end{figure}

	\section{IJ and Neural Network}\label{sec:nn}
	Deep learning methods have recently achieved state-of-the-art performance on a variety of prediction and learning tasks, naturally leading to the need for uncertainty quantification. Mathematically, we can formalize a neural network with $L$ hidden layers and an activation function $\sigma$ as: 
	\begin{align*} \label{eqn:MLP}
		\begin{split}
			\textrm{NN}\left( x; A^{(1)}, b^{(1)}, \ldots, A^{(L)}, b^{(L)} \right) 
			= \; A^{(L)} \sigma\left\{ \ldots A^{(2)} \sigma\left( A^{(1)} x + b^{(1)} \right) \ldots + b^{(L-1)} \right\} + b^{(L)},
		\end{split}	
	\end{align*}
	where the dimension $m_{\ell}$ is the number of nodes at layer $\ell$, $\ell = 0, \ldots, L$, $x \in \mathbb{R}^{m_0}$ is the input signal, $A^{(s)} \in \mathbb{R}^{m_{\ell} \times m_{\ell-1}}, b^{(s)} \in \mathbb{R}^{m_{\ell}}$ are the parameters that produce the linear transformation of the $(\ell-1)$th layer, and the output is a scalar with $m_L=1$.
	
	The most immediate way to approach uncertainty quantification for neural networks is to treat them as parametric models falling within the framework of M-estimators which we explore above. However, as we demonstrate, a number of technical issues lead IJ-based uncertainty quantification to have poor statistical performance.  Firstly, the structure of neural network is much more complex than traditional M-estimator models because of the nonlinear transformation with activation functions and the high dimensions of the parameter space. In addition, it is known that the optimization of neural network is quite challenging [\cite{sun2020optimization}] and often approached by using stochastic gradient descent and early stopping [\cite{bottou1991stochastic}]. Even when they provide good predictive performance, these methods normally fail to guarantee a global optimal solution. Thus the complex procedure of training neural network also introduces another source of procedural variability because of random initialization, mini-batch gradient descent and early stopping [\cite{huang2021quantifying}]. This makes the uncertainty quantification even more difficult.
	
	\subsection{Neural Networks as M-Estimators}
	
	
	In this paper, we will focus on a small neural network model with only one hidden layer and small number of hidden units. Instead of mini-batch gradient descent, we use the whole training dataset to calculate gradient for optimization and run enough epochs to ensure convergence. In addition, we also consider training the network weights from fixed initialization parameters to remove an important source of variability. Under these conditions, we can assume that the global optimal point can be attained and there is no source of procedural variability. We can then regard a small neural network as an M-estimator if a smooth (sigmoid) activation function is used. By using the calculation in \cref{subsec:ij_m}, we can calculate the IJ directional derivatives for neural networks. We explore the empirical performance of this application of IJ below. 
	
	We conduct a simulation experiment with \textsf{Tensorflow2.0} in python. In our experiments, we use both ReLU and sigmoid (logistic) activation functions. Note that ReLU is not a smooth function so it will violate IJ assumptions. However, as the most commonly used activation function, we will still examine the application of IJ. We also vary the number of hidden units in the range $\{1,3,5,10,20\}$ to see how the performance changes with the complexity of the neural network structure. During the process of optimization, we use the whole training dataset to calculate the gradient and the "Adam" algorithm  for optimization [\cite{kingma2014adam}]. The number of epochs is set to be 1000 and the learning rate is 0.01. In our experimental settings, these were found to reliably result in solutions that had converged to a local optimum.  In these experiments we observe that the Hessian matrix $\EE_{\hat{\DDD}} [\nabla_\theta^2 m(\hat{\theta},Z)]$ in equation (\ref{M_ij}) is generally ill-conditioned because of the numerical issues. To resolve this, we add a small identity matrix $0.001 * I$ to it before taking the inverse if its condition number is larger than $10^5$. We use the same settings as in \cref{sec:sim_frame}. We record both Monto Carlo variance (the sample variance across 200 repetitions) and the coverage of $E[\hat{f}(x)]$ for the predictors across the query points in order to investigate the performance of IJ in representing variance and the distribution of neural networks. These results are presented in Figure \ref{fig_nn}. 
	
	\begin{figure}[p!]
		\centering
		\begin{subfigure}[b]{0.9\textwidth}
			\centering
			\includegraphics[width=\textwidth]{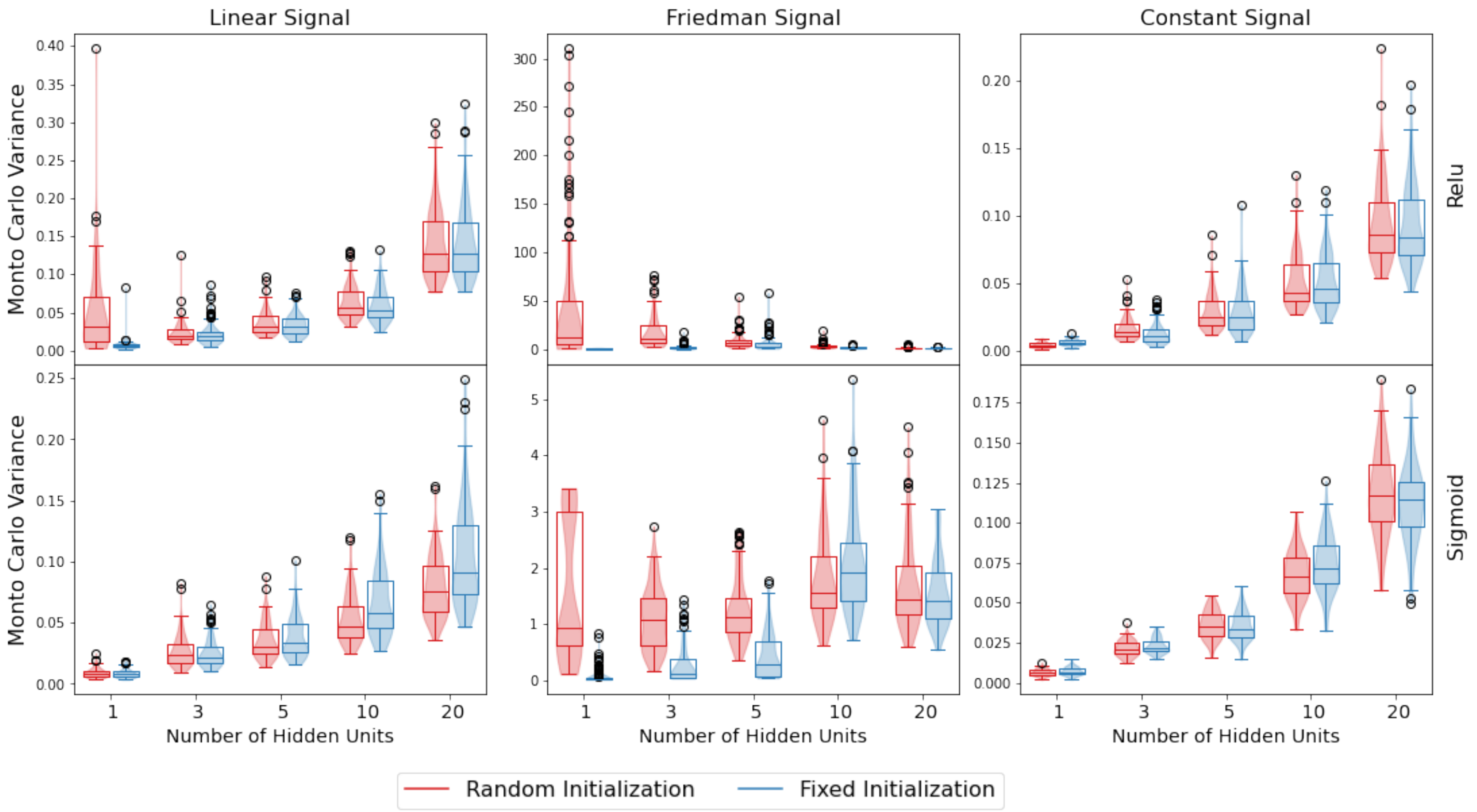}
			\caption{Boxplot and Violin Plot for Monte Carlo Variance.}
		\end{subfigure}
		\hfill
		\begin{subfigure}[b]{0.9\textwidth}
			\centering
			\includegraphics[width=\textwidth]{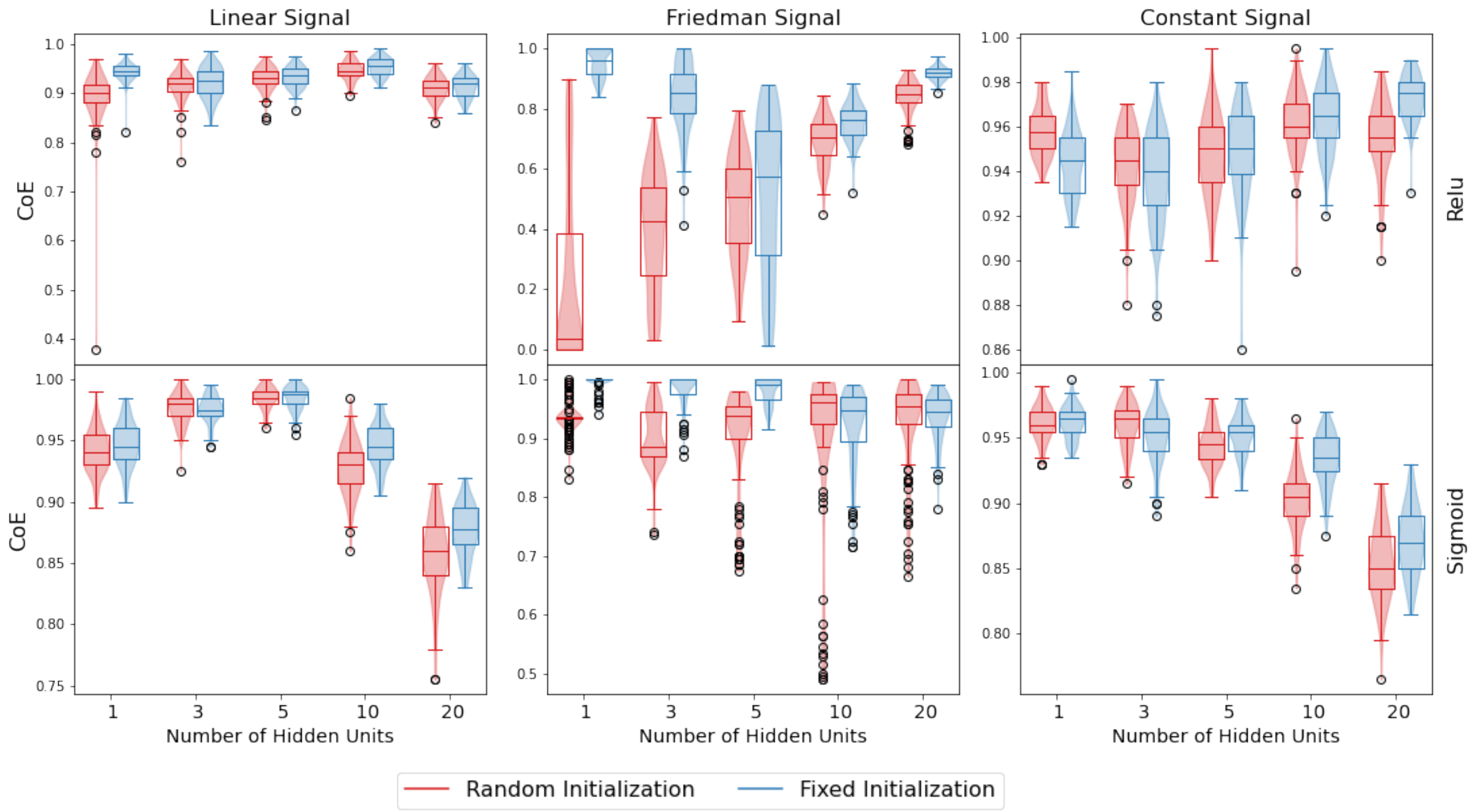}
			\caption{Boxplot and Violin Plot for CoE.}
		\end{subfigure}
		\captionsetup{width=0.9\textwidth}
		\caption{IJ of neural network with different numbers of hidden units. Three data generating processes and both ReLU and sigmoid activation functions are considered and we plot results over 100 query points. (a) gives Monto Carlo Variance, (b) records coverage of $E \hat{f}(x)$. Red indicates the training process with random initialization while blue is for the case when the initialization is fixed with a specific unchanged random seed over the replications. The numbers on the x-axis represents different number of hidden units of the neural network structure.} 
		\label{fig_nn}
	\end{figure}
	
	We first observe that for both ReLU and sigmoid activation functions, the Monto Carlo variance increases as the number of hidden units becomes larger under the Linear and Constant signals. This can be explained as the result of overfitting. In contrast, under the Friedman signal, if we use random initialization and sigmoid function, the Monte Carlo variance is extremely large when the number of hidden units is 1 and actually decreases with the number of hidden units. However, with fixed initialization, the Monte Carlo variance is always quite small. This behavior results from the lack of smoothness of the ReLU function which results in many local optima. As a result, even with very careful optimization procedures, different initializations result in different local optima. When the number of hidden units is increased, the difference of the predictors corresponding to each local optimal point becomes smaller and the procedure variability can be thus reduced. \textit{Starting the optimization from fixed initialization weights removes the effect, revealing initialization to be an important -- and difficult to quantify -- source of variability}. Unfortunately, there do not appear to be good non-random choices of initial weights.  This can also happen when we consider sigmoid activation functions under the Friedman signal. However, the effect is much smaller in this case, since the activation function is smoother, and the optimization process correspondingly less complex. 
	
	Examining the effect of different generating models on converage, it is not surprising to find that the ReLU activation function can achieve very good coverage for both Linear and Constant signals, because this activation function preserves the linear or constant forms of the function resulting in minimal bias. However, a model using ReLU activation exhibits serious undercoverage when we consider random initialization for the reasons discussed above; by fixing the initialization or increasing the number of hidden units, the coverage can be improved a lot. The sigmoid activation function  achieves reasonable coverage under the Friedman signal but suffers from undercoverage when the number of hidden units is increased under Linear and Constant signals. This can be explained by observing that when the number of hidden units is relatively large, many hidden nodes will produce outputs very close to zero or one after the sigmoid transformation. This can lead to an ill-conditioned Hessian matrix, which results in the high inaccuracy of the IJ estimator. 
	
	Overall, we only achieve  valid IJ estimates using the sigmoid activation function when the number of hidden units is quite small and thus the model is close to being parametric; larger networks resulting in numerical challenges. For ReLU activation functions, even though it can achieve promising performance for Linear and Constant signals, it fails to produce satisfactory coverage even for $E[\hat{f}(x)]$ for the Friedman signal. We speculate this is a result of the discontinuities in the derivative of the ReLU activation function; a violation of M-estimator assumptions. For a small network the resulting discontinuities affect only a small enough fraction of the data to be ignored; as the size of the network grows the set of resulting discontinuities becomes dense in the data, rendering variation in the IJ tangent plane a poor approximation to sample variance.  More work is needed to quantify the procedure variability introduced by complex neural network structure and the numerical issue of the ill-conditioned Hessian matrix should also be further carefully studied.

	\subsection{Model Comparisons with Neural Network} 
	
	As a final exercise, we compare neural network with multiple random forests. A small neural network with one hidden layer of 5 units and sigmoid activation function (for which IJ has reasonable performance) is considered. The result is presented in Figure \ref{fig_com_nn}. Here we see an obvious difference between neural network and random forests under Friedman signal while both of them are quite similar under the other two. Here constraining the maximum depth of trees likely leads to additional bias in the random forest and high power to detect differences between the two model classes. 
	
	\begin{figure}[ht!]
		\centering
		\includegraphics[width=0.9\textwidth]{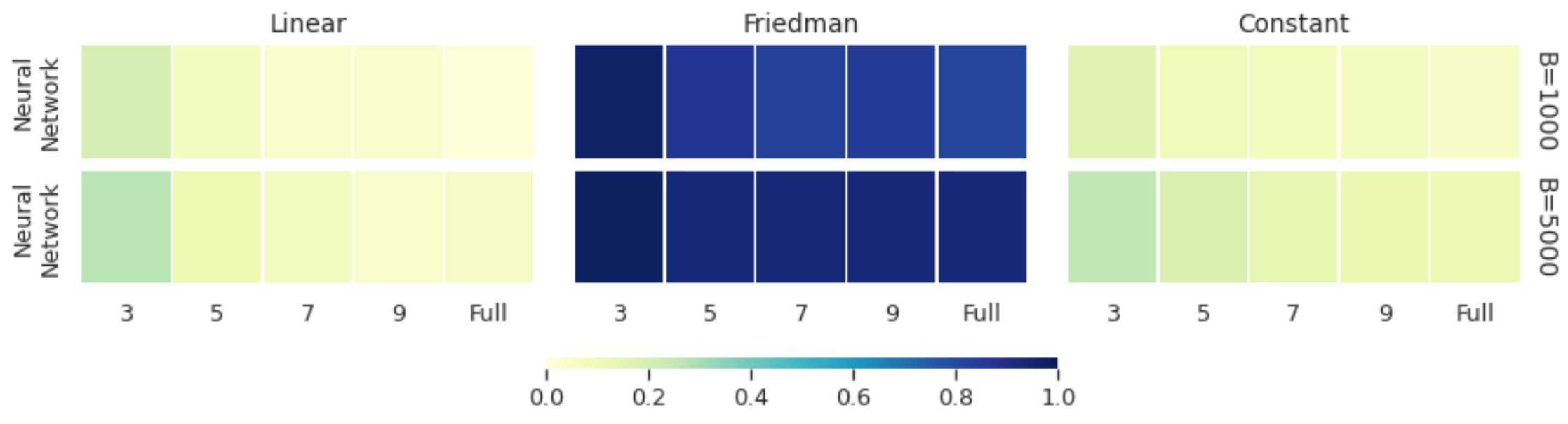}		
		\captionsetup{width=0.9\textwidth}
		\caption{Model Comparisons of multiple random forests and neural network. A small neural network with one hidden layer of 5 units and sigmoid activation function is used. The rest of the setup is the same as Figure \ref{fig_com_rf}.}
		\label{fig_com_nn}
	\end{figure}

	\section{Ensemble Models as V-statistics} \label{sec:v-stat}
	
	Section \cref{subsec:rf_bias} introduced IJ calculations for random forests. In this section we observe that the IJ derivation applies to any ensemble obtained from subsamples taken with replacement.  In particular, this represents an alternative to bootstrap methods for uncertainty quantification; we can employ an ensemble of models which generally exhibit lower variance, and provide uncertainty quantification with little additional cost. Recall that for random forests model, we use $T_b(x)$ to represent the prediction of the tree at the data point $x$, where this tree is trained on the $b$th subsample through bagging process. Generally, we can apply this framework to any other models by denoting $T_b(x)$ as the prediction of the model at the data point $x$, where the model is trained on the $b$\textsuperscript{th} subsample. Following exactly the procedure in \cref{subsec:rf_bias}, we can calculate the covariance matrix and directional derivatives for uncertainty quantification or model comparisons. We advocate for using this approach over bootstrapping for uncertainty quantification (i.e. building one model on the whole data set and using bootstrapped models to produce confidence intervals) under the argument that this allows you to use the ensemble structure to stabilize prediction {\em as well as} provide variance estimates. 
	

	In the following, we use an example of XGBoost for demonstrations. XGBoost (Extreme Gradient Boosting)  is a scalable, distributed gradient-boosted decision tree machine learning method that has gained great popularity and attention recently in the field of machine learning. We obtain an "ensemble XGBoost" model by training multiple XGBoost models, each on a subsample taken with replacement, falling within the framework of V-statistics.  We implemented this using \textsf{xgboost} in python with its default hyperparameter settings: 100 trees each of depth 6, step shrinkage parameter $\eta = 0.3$ and no subsampling within the XGBoost process. Our ensemble is generated from 1000 subsamples, each of size 200. In terms of the notation above each $T_b$ is an XGBoost model. The simulation settings in \cref{sec:sim_frame} were employed to calculate coverage of $E[\hat{f}(x)]$ with confidence intervals from ensemble XGBoost and compare it with multiple random forests models. We plot the results in Figure \ref{fig_xgboost_ci}, which demonstrates a very good performance for the coverage. In Figure \ref{fig_xgboost_comp}, we observe that for the Linear and Constant signals, the difference between ensemble XGBoost and random forests is much less pronounced. For the Friedman signal, the difference is quite obvious, especially when the maximum depth is small. We can expect that as the maximum depth increases, random forest will tend to behave similarily to ensemble XGBoost. 
	
	\begin{figure}[ht!]
		\centering
		\includegraphics[width=0.75\textwidth]{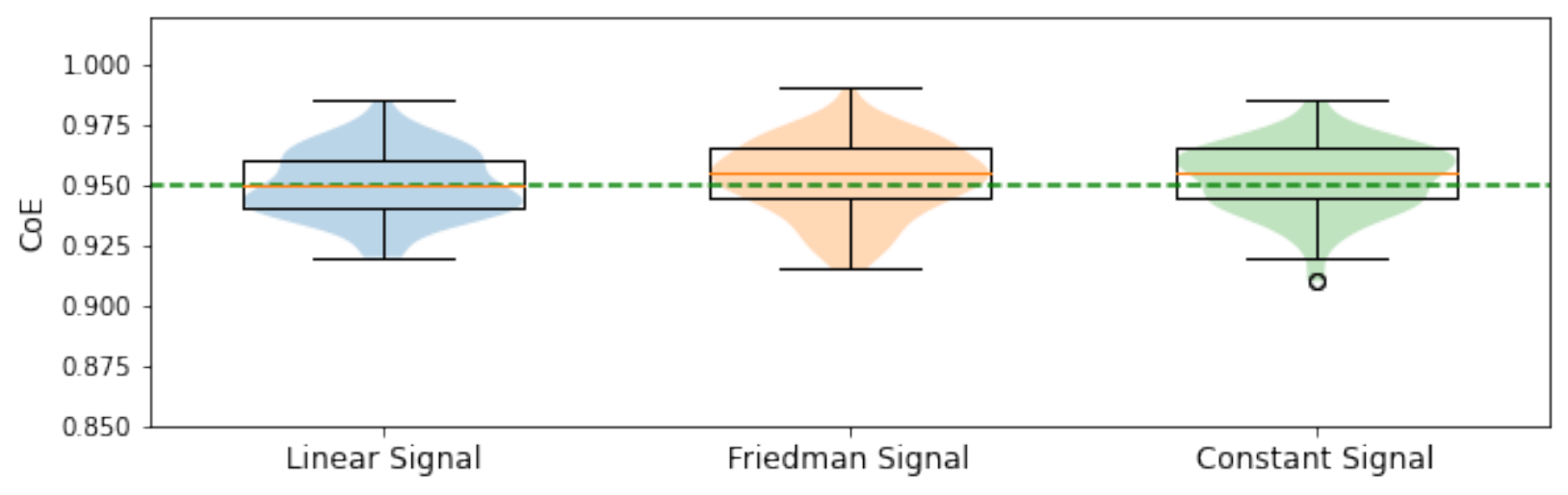}
		\captionsetup{width=0.9\textwidth}
		\caption{Coverage of $E \hat{f}(x)$ using 95\% intervals for ensemble XGBoost over 100 query points. The number of subsamples is set to be 1000 with size 200.} \label{fig_xgboost_ci}
	\end{figure}
	
	\begin{figure}[ht!]
		\centering
		\includegraphics[width=1\textwidth]{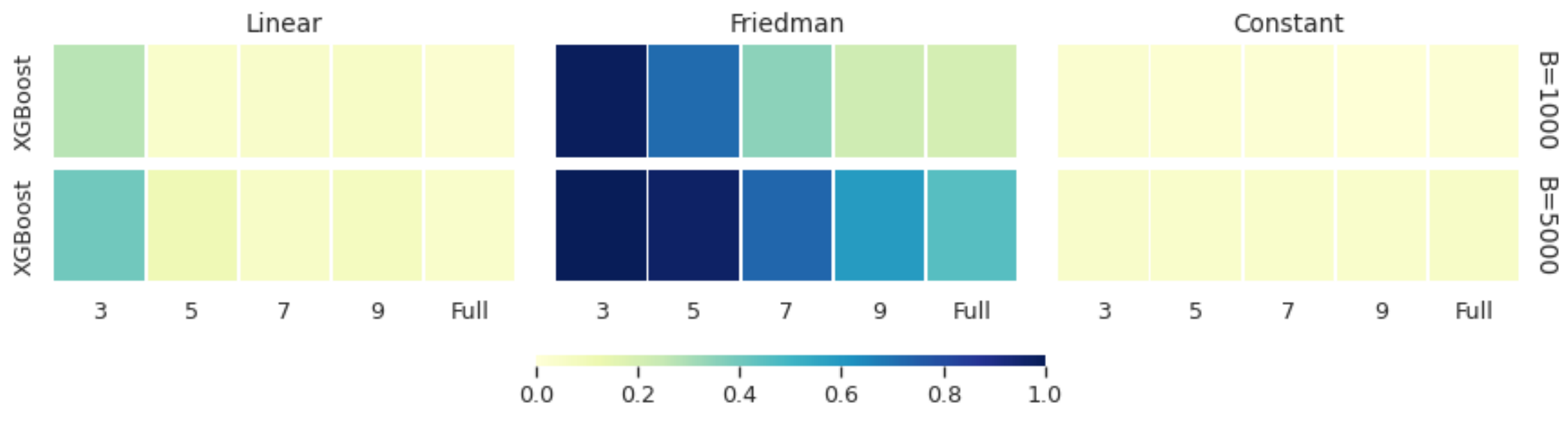}
		\captionsetup{width=0.9\textwidth}
		\caption{Model Comparisons of multiple random forests and ensemble XGBoost. The rest of the setup is the same as Figure \ref{fig_com_rf}.} \label{fig_xgboost_comp}
	\end{figure}

	\section{Real Data Analysis} \label{sec:real}
	
	We study the practical application of our methods to Beijing Housing Data [\cite{house}]. These data contain the price of Beijing houses from 2011 to 2017, (which was scraped from \url{www.lianjia.com}), along with covariates including location, size, building features, age and configuration. We pre-process the data by removing the rows with missing values and dropping the columns of url, transaction and community id, and the average price per square ft (since this can be calculated by the other two columns: total price and the square ft area of house). The resulting dataset has dimension $15360 \times 51$ after converting all categorical features to dummy variables. We use the logarithm of the total price as the outcome to be predicted from the remaining covariates. The size of these data allows us to explore the coverage properties of IJ-derived intervals in a real-world setting. Specifically, we randomly split the whole data into 50 disjoint training sets $\{\mathcal{D}_i\}_{i=1}^{50}$ each of size 3000. We can now build models with associated intervals on each of these independent training sets and examine whether the intervals cover the predictions of independently-generated models. To do so, we reserve the remaining 3260 as a test set from which we randomly draw 100 samples to use as query points on which we examine the performance of confidence intervals. 
	
	For this section, we consider ``reproduction interval'' as used in \cite{zhou2022boulevard}.  Rather than being defined so as to cover a population quantity, these are defined to cover the prediction that would be made by applying the same learning procedure on an independent sample.  That is, if we train models $\hat{f}_1(x), \hat{f}_2(x)$ using the same process on independent data sets such that each is distributed as $N(\EE[\hat{f}_i(x)], V(x))$ 
	then $\hat{f}_1(x) - \hat{f}_2(x) \sim N(0, 2V(x))$. Thus we can define an interval
	\[
	I(x) = \left[ \hat{f}_1(x) - \Phi^{-1}(0.975) \sqrt{ 2 \hat{V}(x)}, \hat{f}_1(x) + \Phi^{-1}(0.975) \sqrt{2 \hat{V}(x)} \right]
	\]
	which should cover 95\% of $\hat{f}_2(x)$ replicates of the same process. 
	
	For our assessment, at each query point we construct a reproduction interval $I_i(x)$ for each $D_i$. We then measure the proportion of replicates falling within $I_i(x)$ and report the {\em Coverage of Reproduction} (CoR) by averaging these proportions:
	\[
	\mbox{CoR}(x) = \frac{1}{50 \times 49} \sum_{i=1}^{50} \sum_{j \neq i} ( \hat{f}_j(x) \in I_i(x) ).
	\]
	
	We also evaluate the mean square error (MSE) for each estimated model $\hat{f}_i$ on the testing set, denoted by $e_i$. Then we report the average MSE $\frac{1}{50} \sum_{i=1}^{50}e_i$ of all the models fitted on each training segment.

	We consider five different models: (1) LM (linear model), (2) LM+RF (boost linear model with a random forest), (3) RF (random forest), (4) RF+RF (boost a random forest with another random forest) (5) XGB (ensemble XGBoost). We use random forests with full depth and 1000 trees here and report the boxplot and violin plot of CoR and MSE of each model. The result is presented in Figure \ref{fig:sub1}. We can observe that all these models achieve sufficient coverage and the models that involve random forests tend to be more conservative for CoR. As a byproduct, for a fixed query point, we plot the confidence intervals of predictions by the models fitted on one specific segment of the training data shown in Figure \ref{fig:sub2}. In addition, we also implement the model comparisons test for each pair of model predictions and use the horizontal lines to  group the pairs of model predictions which are not statistically distinguishable. We observe that the prediction of random forests is relatively similar with ensemble XGBoost, boosted random forests and modified random forests but that linear models do appear to give a statistically distinct prediction, even when later modified by a random forest. 
	
	\begin{figure}[ht!]
		\begin{subfigure}{.5\textwidth}
			\centering
			\includegraphics[width=0.95\textwidth]{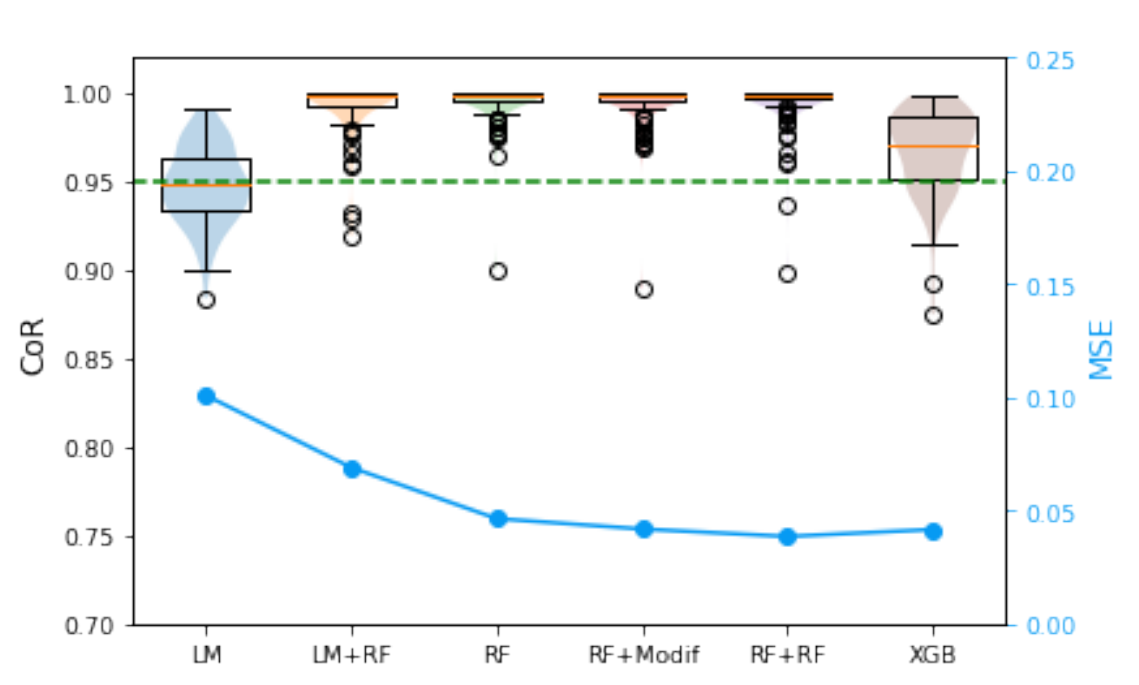}
			\caption{CoR and MSE}
			\label{fig:sub1}
		\end{subfigure}
		\hspace{1em}
		\begin{subfigure}{.4\textwidth}
			\centering
			\includegraphics[width=1\textwidth]{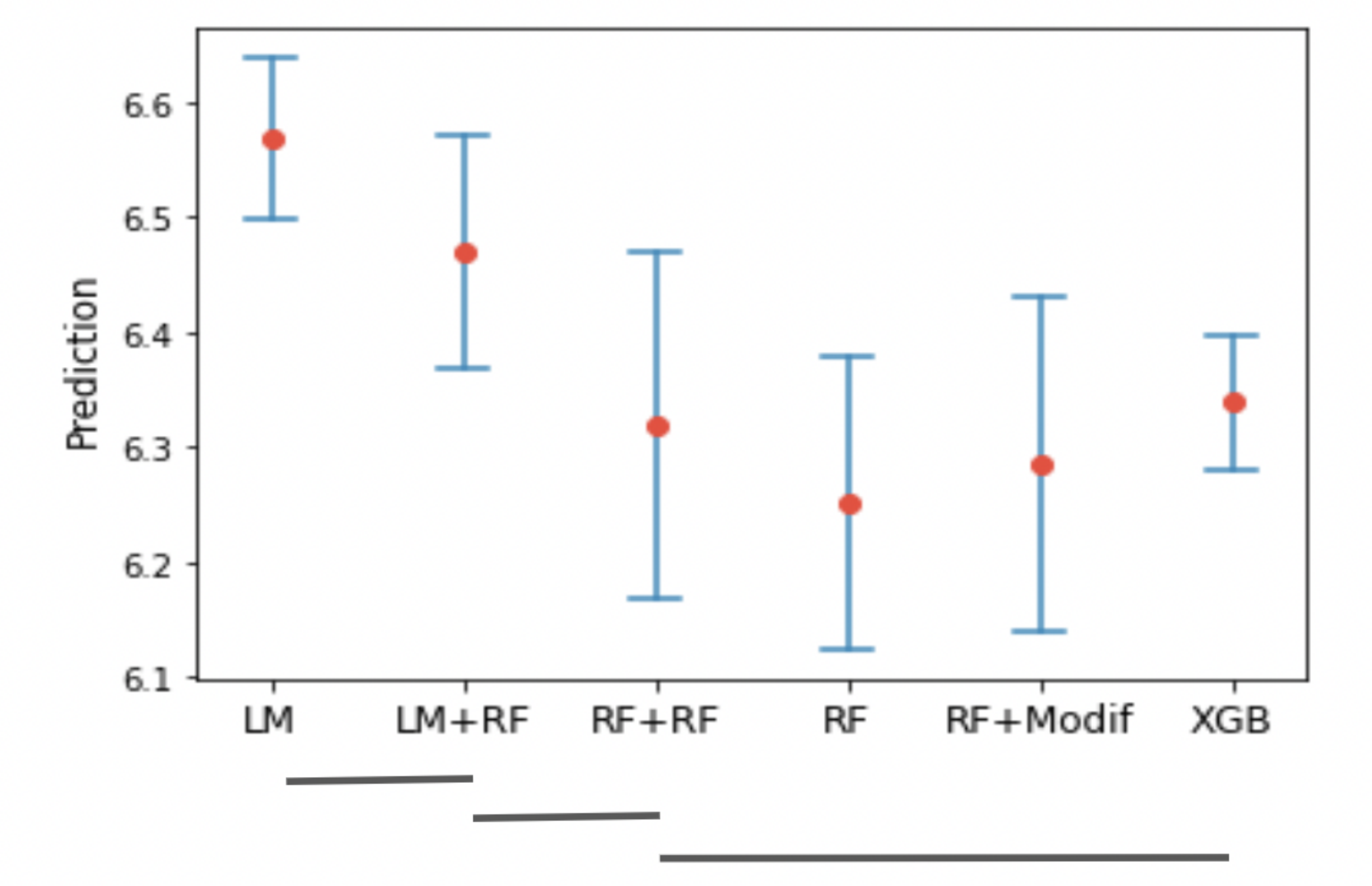}
			\caption{Model Comparisons}
			\label{fig:sub2}
		\end{subfigure}
		\captionsetup{width=0.9\textwidth}
		\caption{Analysis of Beijing Housing Data. (a) boxplot and violin plot of CoR and MSE of each model. x-axis represents five different model combinations: LM (linear model), LM+RF (boost linear model with a random forest),  RF (random forest), RF+RF (boost a random forest with another random forest), and XGB (ensemble XGBoost). (b) confidence intervals of predictions for these different model combinations. The bottom horizontal lines group the pairs of model predictions which are not statistically distinguishable.}
	\end{figure}

	\section{Discussions} \label{sec:dis}
	In this paper we revisited the classical but underused IJ method, first described in \cite{efron1982jackknife} to derive variance estimates. The underlying IJ directional derivatives are highly versatile and can be used to construct estimates of the covariance between two different models trained on the same data. With this we devised a statistical test to quantify how exactly a model differs from another. In the case of ensemble methods, IJ is biased upwards. However bias corrections can be successfully derived for ensembles obtained by subsampling with replacement.

	We demonstrated that IJ is an efficient method for model comparisons or providing uncertainty estimates for
	combinations of models in the case of random forests or M estimators, where the specific expressions for directional derivatives are derived. However, as for more general machine learning models, we propose to emsemble them as a V-statistic, which provides a generic means of producing uncertainty quantification.
	
	There are also some ideas related to the content of this paper that could be explored in the future. Even though we only worked with Gaussian responses in this paper, the methods can be readily applied to non-Gaussian responses from an exponential family. The software corresponding to this paper can work with such data, where the random forest ``boosts'' are fitted on pseudo-residuals as defined in \cite{ghosal2021generalised}.
	
	Closed forms for the IJ could also be explored for any general V-statistic or more general ensemble methods. For any two given models proving consistency of the IJ covariance estimate (or disproving it or constructing suitable regularity conditions) is also a theoretical challenge.
	
	Note that we have ignored the uncertainty in the estimate $\hat{\Sigma}$ throughout. If there are a large number of query points, IJ estimates of the covariance matrix $\Sigma$ could be very unstable. A potential future direction would be to characterize the distribution of this estimate, especially for random forests covariances, and modify our uncertainty quantification correspondingly.
	
	We could also extend the boosting framework discussed in \cref{sec:v-stat} to work with ensembles where the contribution from each boosting stage may be random or vary during the training itself - for example the Boulevard boosting model where the weight for each tree depends on the eventual total (random) number of trees [\cite{zhou2022boulevard}] and it keeps changing during training as more trees are added to the boosting ensemble.
	
	Finally, we also need to resolve the problem of calculating IJ directional derivatives of neural networks. Our simulations show many challenges to a na\"{i}ve framework for this, including the technical issues of optimization, numerical issues of sigmoid activation functions in larger number of hidden units. Overcoming these requires either a more direct calculation of directional derivatives, possibly through automatic differentiation, or a new representation of neural network training.

	\baselineskip=21pt
	\bibliographystyle{chicago}
	\bibliography{paper_v5}
	
	\afterpage{\clearpage}
	\newpage
	
	\appendix
	
	\section{Asymptotics for M-estimators} \label{sec:MestBasics}
	
	\paragraph{The Jacobian} Suppose $f: \RR^{m+n} \to \RR^\ell$ is a differentiable function. Then for any $(x_0,y_0) \in \RR^{m+n}$ (where $x_0 \in \RR^m, y_0 \in \RR^n$) the (partial) Jacobian is a matrix [in $\RR^{\ell \times m}$] defined by
	\begin{align*}
		J_{f,x}(x_0, y_0) &= \begin{pmatrix} \frac{\partial f}{\partial x_1}(x_0, y_0) \cdots \frac{\partial f}{\partial x_m}(x_0, y_0) \end{pmatrix} \\
		&= \begin{pmatrix} (\nabla_x^\top f_1)(x_0, y_0) \\ \vdots \\ (\nabla_x^\top f_\ell)(x_0, y_0) \end{pmatrix} \\
		&= \left(\left(\frac{\partial f_i}{\partial x_j}(x_0,y_0)\right)\right)_{\substack{i=1,\dots,\ell \\ j=1,\dots,m}} 
	\end{align*}
	
	\paragraph{The Delta Method} Suppose $X, \{X_n\}_{n=1}^\infty$ are random variables in $\RR^m$, $c \in \RR^m$ is constant and $r_n \to \infty$ is a sequence such that $r_n(X_n - c) \xrightarrow{d} X$. Also $f: \RR^m \to \RR^\ell$ is a differentiable function with Jacobian $J_g$. Then by the Taylor expansion
	\begin{align*}
		f(X_n) &= f(c) + J_f(\tilde{c}) (X_n - c), \\
		\text{where } &\|\tilde{c} - c\| \leq \|X_n - c\| \\
		\implies r_n(f(X_n) - f(c)) &= J_f(\tilde{c}) \cdot r_n(X_n - c)
	\end{align*}
	Now by Slutsky's theorem $X_n - c = r_n(X_n - c) \cdot \frac1{r_n} \xrightarrow{c} X \cdot 0 = 0$, i.e, $X_n \xrightarrow{d} c \implies \tilde{c} \xrightarrow{d} c \implies J_f(\tilde{c}) \xrightarrow{d} J_f(c)$. Applying Slutsky's theorem again we see that
	\begin{align*}
		J_f(\tilde{c}) \cdot r_n(X_n - c) &\xrightarrow{d} J_f(c) X \\
		r_n(f(X_n) - f(c)) &\xrightarrow{d} J_f(c) X
	\end{align*}
	
	Now suppose we define the M-estimator to be $\hat{\theta} = \argmax_{\theta \in \Theta} \EE_{\hat{\DDD}} [m(\theta, Z)]$ where $\Theta \subseteq \RR^p$ is the parameter space and the data $\{Z_i\}_{i=1}^n$ are independently generated from the distribution $\DDD$ over $\ZZZ \subseteq \RR^d$. We assume $m$ has a continuous second derivative with respect to $\theta$. Also $\hat{\DDD}$ is the empirical distribution over the training data $(Z_i)_{i=1}^n$ from $\ZZZ^n$, thereby making $\EE_{\hat{\DDD}} [m(\theta, Z)] = \frac1{n} \sum_{i=1}^n m(\theta, Z_i)$. Finally let $\theta_0 = \argmax_{\theta \in \Theta} \EE_{\DDD} [m(\theta, Z)]$ be the real value of the M-estimate.
	
	\paragraph{Asymptotic Consistency} Assume the following conditions
	\begin{align*}
		& \sup_{\theta \in \Theta} \left| \EE_{\hat{\DDD}} [m(\theta, Z)] - \EE_{\DDD} [m(\theta, Z)] \right| \xrightarrow{p} 0 \\
		& \sup_{\theta:\|\theta - \theta_0\| > \epsilon} \EE_{\DDD} [m(\theta, Z)] < \EE_{\DDD} [m(\theta_0, Z)], \;\forall\, \epsilon > 0 \\
		& \EE_{\hat{\DDD}}[m(\hat{\theta}, Z)] \geq \EE_{\hat{\DDD}}[m(\theta_0, Z)] - o_P(1)
	\end{align*}
	The first assumption states that the law of large numbers is uniform over $\Theta$ and the second one states that the global extremum $\theta_0$ is well-separated. Then from the second assumption $\forall\, \epsilon > 0$, $\exists\, \delta > 0$ such that 
	\begin{align*}
		&P(\|\hat{\theta} - \theta_0\| \geq \epsilon) \\
		&\leq P(\EE_{\DDD} [m(\theta_0, Z)] - \EE_{\DDD} [m(\hat{\theta}, Z)] > \delta) \\
		&= P\left( (\EE_{\DDD} [m(\theta_0, Z)] - \EE_{\hat{\DDD}} [m(\theta_0, Z)]) + (\EE_{\hat{\DDD}} [m(\theta_0, Z)] - \EE_{\hat{\DDD}} [m(\hat{\theta}, Z)]) \right. \\
		&\qquad\qquad \left. + (\EE_{\hat{\DDD}} [m(\hat{\theta}, Z)] - \EE_{\DDD} [m(\hat{\theta}, Z)]) \geq \delta \right) \\
		&\leq P(\EE_{\DDD} [m(\theta_0, Z)] - \EE_{\hat{\DDD}} [m(\theta_0, Z)] \geq \delta/3) + P(\EE_{\hat{\DDD}} [m(\theta_0, Z)] - \EE_{\hat{\DDD}} [m(\hat{\theta}, Z)] \geq \delta/3) \\
		&\qquad\qquad + P(\EE_{\hat{\DDD}} [m(\hat{\theta}, Z)] - \EE_{\DDD} [m(\hat{\theta}, Z)] \geq \delta/3)
	\end{align*}
	Now it is easily seen that by the first and third terms above has a limit of 0 (by the first assumption) and the second term also has a limit of 0 (by the third assumption). Thus $\forall \, \epsilon > 0$, $P(\|\hat{\theta} - \theta_0\| \geq \epsilon) \to 0$, i.e., $\hat{\theta} \to \theta$ is probability and thus almost everywhere.
	
	We can also do a similar proof of consistency for the corresponding Z-estimators. Define $g(\theta,Z) = \nabla_\theta m(\theta, Z)$ and suppose $\hat{\theta}$ is an (approximate) solution of $\EE_{\hat{\DDD}}[g(\theta, Z)] = 0$ and $\theta_0$ is an (exact) solution of $\EE_{\DDD}[g(\theta, Z)] = 0$ satisfying the following conditions
	\begin{align*}
		& \sup_{\theta \in \Theta} \| \EE_{\hat{\DDD}} [g(\theta, Z)] - \EE_{\DDD} [g(\theta, Z)] \| \xrightarrow{p} 0 \\
		& \sup_{\theta:\|\theta - \theta_0\| > \epsilon} \|\EE_{\DDD} [g(\theta, Z)] \| > 0, \;\forall\, \epsilon > 0 \\
		& \| \EE_{\hat{\DDD}} [g(\hat{\theta}, Z)] - \EE_{\hat{\DDD}} [g(\theta_0, Z)] \| \xrightarrow{p} 0
	\end{align*}
	Then by the second assumption $\forall\, \epsilon > 0$, $\exists\, \delta > 0$ such that
	\begin{align*}
		&P(\|\hat{\theta} - \theta_0\| \geq \epsilon) \leq P(\|\EE_{\DDD} [g(\hat{\theta}, Z)] \| > \delta) \\
		&= P\bigg( (\|\EE_{\DDD} [g(\hat{\theta}, Z)] \| - \|\EE_{\hat{\DDD}} [g(\hat{\theta}, Z)] \|) + (\|\EE_{\hat{\DDD}} [g(\hat{\theta}, Z)] \| - \|\EE_{\hat{\DDD}} [g(\theta_0, Z)] \|) \bigg. \\
		&\qquad\qquad \bigg. + (\|\EE_{\hat{\DDD}} [g(\theta_0, Z)] \| - \|\EE_{\DDD} [g(\theta_0, Z)] \|) \geq \delta \bigg) \\
		&\leq P(\|\EE_{\DDD} [g(\hat{\theta}, Z)] - \EE_{\hat{\DDD}} [g(\hat{\theta}, Z)] \| \geq \delta/3) + P(\|\EE_{\hat{\DDD}} [g(\hat{\theta}, Z)] - \EE_{\hat{\DDD}} [g(\theta_0, Z)] \| \geq \delta/3) \\
		&\qquad\qquad + P(\|\EE_{\hat{\DDD}} [g(\theta_0, Z)] - \EE_{\DDD} [g(\theta_0, Z)] \| \geq \delta/3) \\
		&\qquad\qquad\qquad[\because \|a\|-\|b\| \geq x \implies \|a-b\| \geq x]
	\end{align*}
	Now it is easily seen that by the first and third terms above has a limit of 0 (by the first assumption) and the second term also has a limit of 0 (by the third assumption). Thus $\forall \, \epsilon > 0$, $P(\|\hat{\theta} - \theta_0\| \geq \epsilon) \to 0$, i.e., $\hat{\theta} \to \theta$ is probability and thus almost everywhere.
	
	Note that if we assume $\hat{\theta}$ is always an exact solution of $\EE_{\hat{\DDD}}[g(\theta, Z)] = 0$ then we can drop the third assumption and the proof is also simplified as follows - by the second assumption $\forall\, \epsilon > 0$, $\exists\, \delta > 0$ such that
	\begin{align*}
		&P(\|\hat{\theta} - \theta_0\| \geq \epsilon) \\
		&\leq P(\|\EE_{\DDD} [g(\hat{\theta}, Z)] \| > \delta) \\
		&= P\left( \|\EE_{\DDD} [g(\hat{\theta}, Z)] \| - \|\EE_{\DDD} [g(\theta_0, Z)] \|) \geq \delta \right) \\
		&\leq P(\|\EE_{\DDD} [g(\hat{\theta}, Z)] - \EE_{\DDD} [g(\theta_0, Z)] \| \geq \delta) \qquad[\because \|a\|-\|b\| \geq x \implies \|a-b\| \geq x]
	\end{align*}
	and this term has a limit of 0 by the first condition.

	\paragraph{Asymptotic Normality}
	
	We know that $\hat{\theta} = \argmax_{\theta \in \Theta} \EE_{\hat{\DDD}} [m(\theta, Z)]$ also satisfies the equation $\EE_{\hat{\DDD}} [\nabla_\theta m(\theta, Z)] = 0$. Thus
	\begin{align*}
		0 &= \EE_{\hat{\DDD}} [\nabla_\theta m(\hat{\theta}, Z)] \\
		&= \EE_{\hat{\DDD}} [\nabla_\theta m(\theta_0, Z)] + \EE_{\hat{\DDD}} \left[ J_{\nabla_\theta m, \theta} (\tilde{\theta}, Z) \right] (\hat{\theta} - \theta_0) \\
		& \text{where } \|\tilde{\theta} - \theta_0\| \leq \|\hat{\theta} - \theta_0\| \\
		\implies \hat{\theta} - \theta_0 &= -\left( \EE_{\hat{\DDD}} \left[ J_{\nabla_\theta m, \theta} (\tilde{\theta}, Z) \right] \right)^{-1} \EE_{\hat{\DDD}} [\nabla_\theta m(\theta_0, Z)]
	\end{align*}
	Now $J_{\nabla_\theta m, \theta} = \nabla_\theta^2 m$. Also
	\begin{align*}
		\EE_{\hat{\DDD}} \left[ \nabla_\theta^2 m(\tilde{\theta}, Z) \right] &\xrightarrow{\text{a.e.}} \EE_{\DDD} \left[ \nabla_\theta^2 m(\tilde{\theta}, Z) \right] \qquad [\text{by the uniform strong law of large numbers}] \\
		&\xrightarrow{\text{a.e.}} \EE_{\DDD} \left[ \nabla_\theta^2 m(\theta_0, Z) \right] \qquad [\text{by continuity of } \nabla_\theta^2 m] \\
	\end{align*}
	And by the central limit theorem 
	\begin{align*}
		\sqrt{n} \left( \EE_{\hat{\DDD}} [\nabla_\theta m(\theta_0, Z)] - \EE_{\DDD} [\nabla_\theta m(\theta_0, Z)] \right) \sim N_p(0, Var_{\DDD} [\nabla_\theta m(\theta_0, Z)])
	\end{align*}
	where $\EE_{\DDD} [\nabla_\theta m(\theta_0, Z)] = 0$ and $Var_{\DDD} [\nabla_\theta m(\theta_0, Z)]$ is a $p \times p$ matrix whose elements are 
	\begin{align*}
		\left[Var_{\DDD} [\nabla_\theta m(\theta_0, Z)]\right]_{ij} &= Cov_{\DDD} \left[ \frac{\partial m}{\partial \theta_i}(\theta_0, Z), \frac{\partial m}{\partial \theta_j}(\theta_0, Z) \right] \\
		&= \EE_{\DDD} \left[ \left(\frac{\partial m}{\partial \theta_i} \cdot \frac{\partial m}{\partial \theta_j}\right)(\theta_0, Z) \right], \;\; i,j = 1, \dots, p \\
		\implies Var_{\DDD} [\nabla_\theta m(\theta_0, Z)] &= \EE_{\DDD} [(\nabla_\theta m \nabla_\theta^\top m)(\theta_0, Z)]
	\end{align*}
	Finally we can apply Slutsky's theorem to obtain
	\begin{align*}
		\sqrt{n} (\hat{\theta} - \theta_0) &\xrightarrow{d} N_p(0, \Sigma(\theta_0)), \\
		\text{ where }& \Sigma(\theta) = \left(\EE_{\DDD} \left[ \nabla_\theta^2 m(\theta, Z) \right]\right)^{-1} \EE_{\DDD} [(\nabla_\theta m \nabla_\theta^\top m)(\theta, Z)] \left(\EE_{\DDD} \left[ \nabla_\theta^2 m(\theta, Z) \right]\right)^{-1}
	\end{align*}
	
	If the prediction function corresponding to any parameter $\theta \in \Theta$ and query point $x \in \ZZZ$ is $\eta(\theta, x)$ then by the delta method
	$$
	\sqrt{n} (\eta(\hat{\theta},x) - \eta(\theta_0,x)) \xrightarrow{d} N_p(0, V(x)), \text{ where } V(x) = \nabla_\theta^\top \eta(\theta_0,x) \Sigma(\theta_0) \nabla_\theta \eta(\theta_0,x)
	$$
	
	\section{Directional derivatives for M-estimators} \label{sec:MestIJ}
	
	Given the training dataset $(Z_i)_{i=1}^n$ and any probability vector $P$ in $\RR^n$ define 
	$$
	\EE_{\hat{\DDD}(P)}[m(\theta, Z)] = \sum_{k=1}^n P_k m(\theta, Z_k)
	$$
	Now for some $i \in \{1,\dots,n\}$ and $\epsilon > 0$ if $P = P(i,\epsilon)$ is given by $P_k = (1-\epsilon)\frac1{n} + \epsilon \ind_{k=i}$ then define
	$$
	\hat{\theta}(i,\epsilon) = \argmax_{\theta \in \Theta} \EE_{\hat{\DDD}(P(i,\epsilon))}[m(\theta, Z)]
	$$
	Thus $\hat{\theta}(i,\epsilon)$ satisfies $\EE_{\hat{\DDD}(P(i,\epsilon))}[\nabla_\theta m(\theta, Z)] = 0$. Fixing $i$, if we can show that $\hat{\theta}(i,\epsilon)$ has a derivative at $\theta = 0$ then $U_{\hat{\theta},i} = \frac{\partial \hat{\theta}}{\partial \epsilon} (i,0)$ will be the directional derivative for the M-estimator. To show this suppose $f:\Theta \times \RR_+ \to \RR^p$ is given by
	\begin{align*}
		f(\theta, \epsilon) &= \EE_{\hat{\DDD}(P(i,\epsilon))}[\nabla_\theta m(\theta, Z)] \\
		&= \sum_{k=1}^n P_k \nabla_\theta m(\theta, Z_k) \\
		&= \sum_{k=1}^n \left((1-\epsilon)\frac1{n} + \epsilon \ind_{k=i}\right) \nabla_\theta m(\theta, Z_k)
	\end{align*}
	Now we know that $f(\hat{\theta},0) = 0$ and we will assume that the Jacobian at $(\hat{\theta},0)$, given by $J_{f,\theta}(\hat{\theta},0) = \EE_{\hat{\DDD}} [\nabla_\theta^2 m(\hat{\theta})]$, is invertible. Then by the implicit function theorem there exists an open set $S \ni 0$ and a continuously differentiable function $h:\RR_+ \to \RR^p$ such that $f(h(\epsilon), \epsilon) = 0$ for all $\epsilon \in S$. So we can express $\hat{\theta}(i,\epsilon) = h(\epsilon)$ making $h(0) = \hat{\theta}$. Further the derivative of $h$ is given by $\frac{\partial h}{\partial \epsilon}(\epsilon) = - \left(J_{f,\theta}(h(\epsilon),\epsilon)\right)^{-1} \frac{\partial f}{\partial \epsilon}(h(\epsilon),\epsilon)$ over $S$. Hence when $\epsilon = 0$, we get the directional derivative to be
	\begin{align*}
		U_{\hat{\theta},i} &= \frac{\partial \hat{\theta}}{\partial \epsilon} (i,0) = \frac{\partial h}{\partial \epsilon}(0) = - \left(J_{f,\theta}(h(0),0)\right)^{-1} \frac{\partial f}{\partial \epsilon}(h(0),0) \\
		&= - \left(J_{f,\theta}(\hat{\theta},0)\right)^{-1} \frac{\partial f}{\partial \epsilon}(\hat{\theta},0) = - \left(\EE_{\hat{\DDD}} [\nabla_\theta^2 m(\hat{\theta})]\right)^{-1} \frac{\partial f}{\partial \epsilon}(\hat{\theta},0) \\
		&= - \left(\EE_{\hat{\DDD}} [\nabla_\theta^2 m(\hat{\theta})]\right)^{-1} \left( \sum_{k=1}^n \left(\ind_{k=i} - \frac1{n} \right) \nabla_\theta m(\hat{\theta}, Z_k) \right) \\
		&= - \left(\EE_{\hat{\DDD}} [\nabla_\theta^2 m(\hat{\theta})]\right)^{-1} \left( \nabla_\theta m(\hat{\theta}, Z_i) - \EE_{\hat{\DDD}} \left[ \nabla_\theta m(\hat{\theta}, Z)\right] \right) \\
		&= - \left(\EE_{\hat{\DDD}} [\nabla_\theta^2 m(\hat{\theta})]\right)^{-1} \nabla_\theta m(\hat{\theta}, Z_i)
	\end{align*}
	
	Finally if the prediction function corresponding to any parameter $\theta \in \Theta$ and query point $x \in \ZZZ$ is $\eta(\theta, x)$ then by the delta method we define its directional derivative to be
	$$
	U_i(x) = \nabla_\theta^\top \eta(\hat{\theta}, x) U_{\hat{\theta},i} = - \nabla_\theta^\top \eta(\hat{\theta}, x) \left(\EE_{\hat{\DDD}} [\nabla_\theta^2 m(\hat{\theta})]\right)^{-1} \nabla_\theta m(\hat{\theta}, Z_i)
	$$
	
	\section{Consistency of the IJ} \label{sec:IJcons}
	
	\paragraph{Consistency of the IJ variance estimate}
	
	We see that the IJ variance estimate for $\sqrt{n} \cdot \eta(\hat{\theta}, x)$ is
	\begin{align*}
		&n \cdot \frac1{n^2} \sum_{i=1}^n U_i(x)^2 \\
		&= \frac1{n} \sum_{i=1}^n \nabla_\theta^\top \eta(\hat{\theta}, x) U_{\hat{\theta},i} U_{\hat{\theta},i}^\top \nabla_\theta \eta(\hat{\theta}, x) \\
		&= \frac1{n} \nabla_\theta^\top \eta(\hat{\theta}, x) \left(\sum_{i=1}^n U_{\hat{\theta},i} U_{\hat{\theta},i}^\top\right) \nabla_\theta \eta(\hat{\theta}, x) \\
		&= \frac1{n} \nabla_\theta^\top \eta(\hat{\theta}, x) \bigg(\sum_{i=1}^n \left(\EE_{\hat{\DDD}} [\nabla_\theta^2 m(\hat{\theta})]\right)^{-1} \nabla_\theta m(\hat{\theta}, Z_i) \nabla_\theta^\top m(\hat{\theta}, Z_i) \bigg. \\
		&\qquad\qquad\bigg.\left(\EE_{\hat{\DDD}} [\nabla_\theta^2 m(\hat{\theta})]\right)^{-1} \bigg) \nabla_\theta \eta(\hat{\theta}, x) \\
		&= \nabla_\theta^\top \eta(\hat{\theta}, x) \left(\EE_{\hat{\DDD}} [\nabla_\theta^2 m(\hat{\theta})]\right)^{-1} \left(\frac1{n} \sum_{i=1}^n \nabla_\theta m(\hat{\theta}, Z_i) \nabla_\theta^\top m(\hat{\theta}, Z_i) \right) \\
		&\qquad\qquad \left(\EE_{\hat{\DDD}} [\nabla_\theta^2 m(\hat{\theta})]\right)^{-1}  \nabla_\theta \eta(\hat{\theta}, x) \\
		&= \nabla_\theta^\top \eta(\hat{\theta}, x) \left(\EE_{\hat{\DDD}} [\nabla_\theta^2 m(\hat{\theta})]\right)^{-1} \EE_{\hat{\DDD}} [(\nabla_\theta m \nabla_\theta^\top m)(\hat{\theta}, Z)] \left(\EE_{\hat{\DDD}} [\nabla_\theta^2 m(\hat{\theta})]\right)^{-1}  \nabla_\theta \eta(\hat{\theta}, x)
	\end{align*}
	Hence by the strong law of large numbers and that $\hat{\theta} \xrightarrow{\text{a.e.}} \theta_0$ it immediately follows that the IJ variance estimate for M-estimator predictions is consistent. 

	\paragraph{Consistency of the IJ covariance estimate}
	
	Suppose we train an M-estimator $\hat{\theta}$ and a random forest $\hat{F}$ on the same training data - then we look at its predictions at two different points. Suppose these predictions are given by $\eta(\hat{\theta}, x_1)$ and $\hat{F}(x_2)$ respectively. Below we show that the IJ covariance estimate for these predictions are consistent.
	
	\textit{Simplifying the M-estimator prediction}: Note that by Taylor expansion
	$$
	\eta(\hat{\theta}, x_1) = \eta(\theta_0, x_!) + \nabla_\theta^\top \eta(\tilde{\theta}, x_1) (\hat{\theta} - \theta_0), \text{ where } \|\tilde{\theta} - \theta_0\| \leq \|\hat{\theta} - \theta_0\|
	$$
	But $\hat{\theta} \xrightarrow{\text{a.e.}} \theta_0 \implies \tilde{\theta} \xrightarrow{\text{a.e.}} \theta_0$. Also from the proof of the asymptotic normality for $\hat{\theta}$ we saw that 
	\begin{align*}
		\hat{\theta} - \theta_0 &- \left( - \left(\EE_{\DDD} \left[\nabla_\theta^2 m(\theta_0, Z)\right] \right)^{-1} \EE_{\hat{\DDD}}[\nabla_\theta m(\theta_0, Z)] \right) \xrightarrow{\text{a.e.}} 0 \\
		\implies \eta(\hat{\theta}, x_1) &- \left(\eta(\theta_0,x_1) + v_{\DDD}(x_1)^\top \EE_{\hat{\DDD}}[\nabla_\theta m(\theta_0, Z)] \right) \xrightarrow{\text{a.e.}} 0,\\
		\text{ where } &v_{\DDD}(x) = -\nabla_\theta^\top \eta(\theta_0, x) \left(\EE_{\DDD} \left[\nabla_\theta^2 m(\theta_0, Z)\right] \right)^{-1}
	\end{align*}
	
	\textit{Simplifying the random forest prediction}: Now consider the tree kernel $T$ of the random forest and its first Hajek projection given by $\mathring{T}(x; Z_1, \dots, Z_k) = \EE[T] + \sum_{j=1}^k T_1(x; Z_j)$. Then the first Hajek projection of the random forest is given by 
	$$
	\mathring{F}(x_2) = \EE[T] + \frac{k}{n} \sum_{j=1}^n T_1(x_2; Z_j).
	$$
	Now suppose the subsample size of the training data for each tree $k = k_n$ varies with $n$ such that
	\begin{align*}
		\lim_{n \to \infty} &\frac{k_n}{n} \to 0 \text{ and for all query points } x, \lim_{n \to \infty} \frac{k_n \zeta_{1,k_n}(x)}{\zeta_{k_n,k_n}(x)} \neq 0, \\
		\text{ where } \zeta_{c,k}(x) &= cov(T(x; Z_1, \dots, Z_c, Z_{c+1}, \dots, Z_k), T(x; Z_1, \dots, Z_c, Z'_{c+1}, \dots, Z'_k)) \\
		&= var(\EE[T(x; Z_1, \dots, Z_k) \mid Z_1 = z_1, \dots, Z_c = z_c])
	\end{align*}
	Then this condition, alongwith Lemma 13 in \cite{wager2018estimation}, shows that we can replace $\hat{F}$ with $\mathring{F}$ without affecting the variance of the random forest. It can also be easily extended to the covariance of the random forest with any other well-behaved model. Note that this condition was also used in Theorem 1 of \cite{ghosal2020boosting}. We can also use a slightly weaker condition
	$$
	\lim_{n \to \infty} \frac{k_n}{n} \cdot \frac{\zeta_{k_n,k_n}(x)}{k_n \zeta_{1,k_n}(x)} = 0,
	$$
	as seen in Theorem 1 of \cite{peng2022rates}.
	
	\textit{Simplifying the theoretical covariance}: Since $\hat{F}$ is replaceable with $\mathring{F}$, it follows that
	\begin{align*}
		&cov_\DDD(\eta(\hat{\theta}, x_1), \hat{F}(x_2)) \\
		&= cov_\DDD(\eta(\hat{\theta}, x_1), \mathring{F}(x_2)) \\
		&= cov_\DDD\left(v_{\DDD}(x_1)^\top \left[\frac1{n} \sum_{j=1}^n \nabla_\theta m(\theta_0, Z_j)\right], \frac{k}{n} \sum_{j=1}^n T_1(x_2; Z_j)\right) \\
		&= \frac1{n} \cdot \frac{k}{n} \cdot n \cdot cov\left(v_{\DDD}(x_1)^\top \nabla_\theta m(\theta_0, Z_n), T_1(x_2; Z_n)\right) \\
		&= \frac{k}{n} \cdot \EE_{\DDD}[v_{\DDD}(x_1)^\top \nabla_\theta m(\theta_0, Z_n) \cdot T_1(x_2; Z_n)], \;[\text{since } \EE_{\DDD}[T_1] = 0]
	\end{align*}
	
	\textit{Decomposing the M-estimator directional derivative}: Define the empirical version of $v_{\DDD}$ to be $v_{\hat{\DDD}}(x) = -\nabla_\theta^\top \eta(\hat{\theta}, x) \left(\EE_{\hat{\DDD}} \left[\nabla_\theta^2 m(\hat{\theta}, Z)\right] \right)^{-1}$. Then we decompose the directional derivative for $\eta(\hat{\theta}, x_1)$, given by $U_i(x_1) = v_{\hat{\DDD}}(x_1)^\top \nabla_\theta m(\hat{\theta}, Z_i)$ as $U_i(x_1) = B_i(x_1) + S_i(x_1)$, where $B_i(x) = v_{\DDD}(x)^\top \nabla_\theta m(\theta_0, Z_i)$.
	
	\textit{Decomposing the random forest directional derivative}: We know that for a random forest the directional derivatives are given by $U_i'(x_2) = n \cdot cov_b(N_{i,b}, T_b(x_2))$ where the covariance is over $b = 1, \dots, B$ trees, $N_{i,b}$ is the number of times the $i$\textsuperscript{th} training datapoint is in the $b$\textsuperscript{th} tree and $T_b$ is the $b$\textsuperscript{th} tree kernel. Now
	\begin{align*}
		U_i'(x_2) &= n \cdot cov_b(N_{i,b}, T_b(x_2)) \\
		&= n \left( \EE_{V \sim \hat{\DDD}}[\hat{F}(x_2) \mid V_1 = Z_i] - \EE_{V \sim \hat{\DDD}}[\hat{F}(x_2)] \right) \\
		&= k \left( \EE_{V \underset{\sim}{\subset} \hat{\DDD}}[T(x_2) \mid V_1 = Z_i] - \EE_{V \underset{\sim}{\subset} \hat{\DDD}}[T(x_2)] \right) \\
		&= k (A_i(x_2) + R_i(x_2)), \text{ where } A_i(x) = \EE_{V \underset{\sim}{\subset} \hat{\DDD}}[\mathring{T}(x) \mid V_1 = Z_i] - \EE_{V \underset{\sim}{\subset} \hat{\DDD}}[\mathring{T}(x)]
	\end{align*}
	Here $V \underset{\sim}{\subset} \hat{\DDD}$ denotes subsampling (without replacement) from the empirical distribution.
	
	Finally we know the following:
	\begin{align}
		\EE\left[ \frac{k}{n^2} \sum_{i=1}^n B_i(x_1) A_i(x_2) \right] &- cov_\DDD(\eta(\hat{\theta}, x_1), \mathring{F}(x_2)) \xrightarrow{\text{a.e.}} 0 \qquad (\text{Lemma } \ref{lem:covcons} \text{ below}) \label{eqn:covcons}\\
		\frac1{n^2} \sum_{i=1}^n S_i(x_1)^2 &\xrightarrow{p} 0 \qquad (\text{Lemma } \ref{lem:m_est_extra} \text{ below}) \label{eqn:m_est_extra}\\
		\frac{k^2}{n^2} \sum_{i=1}^n R_i(x_2)^2 &\xrightarrow{p} 0 \qquad (\text{Lemma 13, \cite{wager2018estimation}}) \nonumber
	\end{align}
	
	Thus using the Cauchy-Schwartz inequality we see that
	$$
	\EE\left[ \frac{k}{n^2} \sum_{i=1}^n U_i(x_1) U'_i(x_2) \right] - cov_\DDD(\eta(\hat{\theta}, x_1), \mathring{F}(x_2)) \xrightarrow{\text{a.e.}} 0
	$$
	thus proving that the IJ covariance estimate is consistent

	\begin{lem} \label{lem:covcons}
		Equation \eqref{eqn:covcons} holds
	\end{lem}
	
	\begin{proof}
		We know from \cite{wager2018estimation} that
		$$
		A_i(x) = \frac{n-k}{n} \left[ T_1(x; Z_i) - \frac1{n-1} \sum_{j \neq i} T_1(x; Z_j) \right]
		$$
		Thus
		\begin{align*}
			\EE[B_i(x_1) A_i(x_2)] &= \frac{n-k}{n} \cdot \EE[v_{\DDD}(x_1)^\top \nabla_\theta m(\theta_0, Z_i) \cdot T_1(x_2; Z_i)] \\
			\EE\left[ \frac{k}{n^2} \sum_{i=1}^n B_i(x_1) A_i(x_2) \right] &= \frac{k}{n^2} \cdot n \cdot \frac{n-k}{n} \cdot \EE[v_{\DDD}(x_1)^\top \nabla_\theta m(\theta_0, Z_n) \cdot T_1(x_2; Z_n)] \\
			\implies \EE\left[ \frac{k}{n^2} \sum_{i=1}^n B_i(x_1) A_i(x_2) \right] &- cov_\DDD(\eta(\hat{\theta}, x_1), \mathring{F}(x_2)) \xrightarrow{\text{a.e.}} 0 \qquad \left[\because 1 - \frac{k}{n} \to 1 \right]
		\end{align*}
	\end{proof}
	
	\begin{lem} \label{lem:m_est_extra}
		Equation \eqref{eqn:m_est_extra} holds
	\end{lem}
	
	\begin{proof}
		We use $v_{\hat{\DDD}}$ and $v_{\DDD}$ without its argument $x_1$ to reduce notational clutter. Note that $S_i(x_1) = v_{\hat{\DDD}}^\top \nabla_\theta m(\hat{\theta}, Z_i) - v_{\DDD}^\top \nabla_\theta m(\theta_0, Z_i)$. Thus
		\begin{align*}
			&\frac1{n^2} \sum_{i=1}^n S_i(x_1)^2 \\
			&= \frac1{n^2} \sum_{i=1}^n \left(v_{\hat{\DDD}}^\top \nabla_\theta m(\hat{\theta}, Z_i) - v_{\DDD}^\top \nabla_\theta m(\theta_0, Z_i)\right)^2 \\
			&= \frac1{n^2} \sum_{i=1}^n \left( (v_{\hat{\DDD}} - v_{\DDD})^\top \nabla_\theta m(\hat{\theta}, Z_i) - v_{\DDD}^\top \left( \nabla_\theta m(\hat{\theta}, Z_i) - \nabla_\theta m(\theta_0, Z_i) \right) \right)^2 \\
			&= \frac1{n^2} \sum_{i=1}^n 2\left( \left( (v_{\hat{\DDD}} - v_{\DDD})^\top \nabla_\theta m(\hat{\theta}, Z_i) \right)^2 + \left(  v_{\DDD}^\top( \nabla_\theta m(\hat{\theta}, Z_i) - \nabla_\theta m(\theta_0, Z_i) ) \right)^2 \right) \\
			&\leq \frac2{n^2} \sum_{i=1}^n \|v_{\hat{\DDD}} - v_{\DDD}\|^2 \|\nabla_\theta m(\hat{\theta}, Z_i)\|^2 + \frac2{n^2} \sum_{i=1}^n \|v_{\DDD}\|^2 \|\nabla_\theta m(\hat{\theta}, Z_i) - \nabla_\theta m(\theta_0, Z_i) \|^2 \\
			&= \frac{2\|v_{\hat{\DDD}} - v_{\DDD}\|^2}{n} \cdot \EE_{\hat{\DDD}} \left[\|\nabla_\theta m(\hat{\theta}, Z)\|^2 \right] + \frac{2\|v_{\DDD}\|^2}{n^2} \sum_{i=1}^n \|\nabla_\theta^2 m(\tilde{\theta}_i, Z_i) (\hat{\theta} - \theta_0) \|^2 \\
			&\qquad [\text{ where } \|\tilde{\theta}_i - \theta_0\| \leq \|\hat{\theta} - \theta_0\|] \\
			&\leq \frac{2\|v_{\hat{\DDD}} - v_{\DDD}\|^2}{n} \cdot \EE_{\hat{\DDD}} \left[\|\nabla_\theta m(\hat{\theta}, Z)\|^2 \right] + \frac{2\|v_{\DDD}\|^2}{n^2} \sum_{i=1}^n \|\nabla_\theta^2 m(\tilde{\theta}_i, Z_i)\|_1^2 \cdot \|\hat{\theta} - \theta_0\|^2 \\
			&= \frac{2\|v_{\hat{\DDD}} - v_{\DDD}\|^2}{n} \cdot \EE_{\hat{\DDD}} \left[\|\nabla_\theta m(\hat{\theta}, Z)\|^2 \right] + \frac{2\|v_{\DDD}\|^2\cdot \|\hat{\theta} - \theta_0\|^2}{n} \cdot \EE_{\hat{\DDD}} \left[ \|\nabla_\theta^2 m(\tilde{\theta}, Z)\|_1^2 \right]
		\end{align*}
		Now it is easily seen that
		\begin{align*}
			\|v_{\hat{\DDD}} - v_{\DDD}\| &\xrightarrow{p} 0 \\
			\|\hat{\theta} - \theta_0\|^2 &\xrightarrow{p} 0 \\
			\EE_{\hat{\DDD}} \left[\|\nabla_\theta m(\hat{\theta}, Z)\|^2 \right] &\xrightarrow{a.e.} \EE_{\DDD} \left[\|\nabla_\theta m(\theta_0, Z)\|^2 \right] \\
			\EE_{\hat{\DDD}} \left[ \|\nabla_\theta^2 m(\tilde{\theta}, Z)\|_1^2 \right] &\xrightarrow{a.e.} \EE_{\DDD} \left[ \|\nabla_\theta^2 m(\theta_0, Z)\|_1^2 \right]
		\end{align*}
		where the last two limits are due to continuity and the uniform strong law of large numbers. Putting them together we see that $\frac1{n^2} \sum_{i=1}^n S_i(x_1)^2 \xrightarrow{p} 0$.
	\end{proof}
	
	\section{Local linear bias corrections for non-Gaussian responses}
	
	The local linear bias correction proposed in \cite{lu2021unified} and explored in \cref{subsec:modif} will not be suitable to use for responses which are not continuous. But if the response is from a general exponential family then we could use techniques from \cite{ghosal2021generalised} to define a local modification - specifically the idea of looking at squared error from the perspective of log-likelihoods. First note that the $\widehat{Bias}(x)$ defined in \cref{subsec:modif} can be rewritten as
	\begin{align*}
		\widehat{Bias}(x) &= \argmax_b \sum_{k=1}^n w_k(x) \left[- \left(Y_k - (\hat{f}(X_k) + b) \right)^2\right],\\
		w_k(x) &= \sum_{b=1}^B \ind\left\{ Z_k \notin I_b, X_k \in L_b(x) \right\}
	\end{align*}
	Here $w_k(x)$, the out-of-bag weight, is the numerator of $v_k(x)$ as defined in \cref{subsec:modif}, i.e., $v_k(x) = \frac{w_k(x)}{\sum_{\ell=1}^n w_\ell(x)}$. For general responses the goal would be to maximise a similar quantity involving the log-likelihood - suppose we've fitted an initial estimator (an MLE-type estimate or a GLM) $\hat{f}_1$ and then we've fitted a random forest $\hat{f}_2$ based on generalised residuals as defined in \cite{ghosal2021generalised}. Then for a query point $x$ we define the local bias correction to be
	$$
	\widehat{Bias}(x) = \argmax_b \sum_{k=1}^n w_k(x) \ell\left( \hat{f}_1(X_k) + \hat{f}_2(X_k) + b ; Y_k, X_k \right),
	$$
	where $(Y_k, X_k)$ is the $k$\textsuperscript{th} training data and $\ell$ is the log-likelihood function with arguments in the link-space. Then the final link-space prediction for the query point $x$ will be $\hat{f}_1(x) + \hat{f}_2(x) + \widehat{Bias}(x)$. This optimisation step may not always have a closed form solution and could be computationally expensive to implement for each query point separately (compared to the case for continuous responses). Also, theoretically the local bias will behave as an M-estimator so we could use the directional derivatives derived above to quantify uncertainty of this bias - hence enabling construction of confidence intervals and comparison tests with the unmodified predictor $\hat{f}_1 + \hat{f}_2$.
	
\end{document}